\documentclass{article}

\PassOptionsToPackage{numbers, compress}{natbib}

\pdfoutput=1
\usepackage[preprint]{neurips_2024}

\usepackage[utf8]{inputenc} 
\usepackage[T1]{fontenc}    
\usepackage{hyperref}       
\usepackage{url}            
\usepackage{booktabs}       
\usepackage{amsfonts}       
\usepackage{nicefrac}       
\usepackage{microtype}      
\usepackage{xcolor}         
\usepackage{xspace}
\usepackage{amsmath}
\usepackage{graphicx}
\usepackage[algo2e,ruled,linesnumbered]{algorithm2e}
\usepackage{stfloats}
\usepackage{subfiles}
\usepackage{wrapfig}
\usepackage{amssymb}
\usepackage{subfig}

\usepackage{algorithm}
\usepackage{algorithmic}
\usepackage{hyperref}

\usepackage{amsthm}
\usepackage[english]{babel}
\newtheorem{theorem}{Theorem}[section]

\newtheorem{lemma}[theorem]{Lemma}

\newcommand{\alg}{SFDDM\xspace}

\newif\ifshowcomments
\showcommentstrue
\ifshowcomments
\newcommand{\mynote}[2]{\fbox{\bfseries\sffamily\scriptsize{#1}}
{\small$\blacktriangleright$\textsf{\emph{#2}}$\blacktriangleleft$}}
\else
\newcommand{\mynote}[2]{}
\fi

\definecolor{amber}{rgb}{1.0, 0.49, 0.0}

\title{SFDDM: Single-fold Distillation for Diffusion models}

\author{%
  Chi Hong \\
  Delft University of Technology \\
  Delft, Netherlands \\
  \texttt{c.hong@tudelft.nl} \\
  \And
  Jiyue Huang \\
  Delft University of Technology \\
  Delft, Netherlands \\
  \texttt{j.huang-4@tudelft.nl} \\
  \AND
  Robert Birke \\
  University of Torino\\
  Turin, Italy\\
  \texttt{robert.birke@unito.it} \\
  \And
  Dick Epema \\
  Delft University of Technology \\
  Delft, Netherlands \\
  \texttt{D.H.J.Epema@tudelft.nl} \\
  \And
  Stefanie Roos \\
  RPTU Kaiserslautern \\
  Kaiserslautern, Germany \\
  \texttt{stefanie.roos@cs.rptu.de} \\
  \And
  Lydia Y. Chen \\
  University of Neuchatel\\ 
  Neuchatel, Switzerland \\
  \texttt{lydiaychen@ieee.org} \\
}

\begin{document}

\maketitle

\begin{abstract}
  While diffusion models effectively generate remarkable synthetic images, a key limitation is the inference inefficiency, requiring numerous sampling steps. To accelerate inference and maintain high-quality synthesis, teacher-student distillation is applied to compress the diffusion models in a progressive and binary manner by retraining, e.g., reducing the 1024-step model to a 128-step model in 3 folds. In this paper, we propose a single-fold distillation algorithm, \alg, which can flexibly compress the teacher diffusion model into a student model of any desired step, based on reparameterization of the intermediate inputs from the teacher model. To train the student diffusion, we minimize not only the output distance but also the distribution of the hidden variables between the teacher and student model. Extensive experiments on four datasets demonstrate that our student model trained by the proposed \alg is able to sample high-quality data with steps reduced to as little as approximately 1\%, thus, trading off inference time. Our remarkable performance highlights that \alg effectively transfers knowledge in single-fold distillation, achieving semantic consistency and meaningful image interpolation.
\end{abstract}

\section{Introduction}

\label{sec:introduction}


Diffusion models~\cite{song2020score, ho2022cascaded, ddpm:ho2020denoising} have emerged as generative models for images of exceptionally high quality without the necessity of conducting adversarial training.
A diffusion model constitutes a Markov chain of forward steps of slowly adding random noise to data, followed by a reverse denoising process that gradually reconstructs the data from the noise via trained neural networks.
These underlying networks typically use the UNet architecture to better connect the forward steps to the corresponding denoising step. However, 
such models require large numbers of sampling steps, e.g., 1000 in DDPM~\cite{ddpm:ho2020denoising}, which leads to high sampling/inference\footnote{We interchangly use sampling and inference time.} times, limiting their applications in latency-sensitive applications. 

Prior art explored diverse directions to reduce the sampling time of diffusion models and maintain the image synthesis quality. One approach is to reduce the computing complexity by compressing the UNet~\cite{li2023snapfusion, zhao2023mobilediffusion} and leverage the acceleration technologies of modern GPUs. Another approach is to reduce the number of sampling steps, i.e., the required UNet inferences.
\cite{ddim:conf/iclr/SongME21} skips intermediate steps by generalizing the original Markovian process via a class of non-Markovian diffusion steps. These two approaches do not change the original training procedure.
Differently, progressive distillation~\cite{progress:salimans2022progressive,meng2023distillation} introduces a teacher-student framework to reduce a trained teacher diffusion model of $T$ steps into a student diffusion of $T'$ steps, where $T' \ll T$, via multiple binary foldings by retraining. For instance, distilling a 1024-step diffusion model into a 128-step model needs first to train a 512-step intermediate model, then a 256-step, to finally arrive to the 128-step model. The distillation objective is to minimize the output differences between the teacher and student models. Progressive distillation better maintains the synthesis quality than step-skipping, but it incurs high distillation time and 
must comply with specific values of 
$T$ and $T'$ due to progressive halvings.


Our objective is to design an effective distillation method that achieves high quality in (any) small sampling step and concurrently incurs low distillation time.
We propose \alg, a single-fold distillation framework able to reduce a $T$-step teacher diffusion model into a $T'$-step student diffusion model in a single fold. 
Thanks to its single-fold nature, \alg offers the flexibility to distill the teacher model into a student with any number of steps. To such an end, we first define a new student model, which can extract knowledge of the teacher diffusion model, by an arbitrary steps sub-sequence. Our forward process definition solves the challenge of aligning the variables of teacher and student models, enabling flexible single-fold distillation. Secondly, when training the reverse denoise process of the student model, we minimize not only the difference in the model outputs but also in the hidden variables at each 
step to better preserve the image synthesis quality.

To demonstrate the effectiveness, we evaluate \alg against sampling-skip~\cite{ddim:conf/iclr/SongME21} and progressive distillation~\cite{progress:salimans2022progressive} on CIFAR-10, CelebA, LSUN-Church, and LSUN- Bedroom~datasets. We compare their synthesis quality in terms of Fr\'{e}chet Inception Distance (FID)~\cite{fid:conf/nips/HeuselRUNH17} of distilling a 1024-step diffusion model into 128-step and 16-step models. \alg achieves the lowest FID as well as the best perceptual quality, also with flexibility on the numerical relation between the teacher and the student, i.e., works for both 1024 to 128 or 100 steps. Further, the distillation effectiveness is validated on semantic input-output consistency and image interpolation.

We summarize the contributions of this paper as follows: 

\begin{itemize}
    \item 
    We propose a novel single-fold distillation algorithm, \alg, which can agilely  compress teacher diffusion into a student diffusion model of any step in one fold.
    \item We define a new forward process for student diffusion, which aligns and approximates student and teacher Markovian variables, enabling flexible single-fold distillation.
    \item We design effective training for student diffusion by minimizing the difference of output and hidden variables with respect to the teacher diffusion.  
    \item We experimentally demonstrate superiority of \alg in achieving high-quality sampling data by as little as approximately 1\% number of steps.
\end{itemize}

\section{Related studies} 

\label{sec:preliminary}


Diffusion models 
first step-wise destroy in a forward process the training data structure and then learn how to restore the data structure from noise in a reverse process. 
DDPM~\cite{ddpm:ho2020denoising} proposed the first stable and effective implementation capable of high-quality image synthesis. 
However, diffusion models, including DDPM, suffer from slow inference stemming from immense intermediate hidden variables, each the size of the synthetic output, as well as the complex architecture. This sparked research on how to accelerate data synthesis with related work exploring three main directions.

{\bf Fast sampling}. 
Diffusion models mostly rely on an UNet~\cite{unet} architecture combining cross-attention and ResNet blocks for denoising.
Fast sampling
~\cite{li2023snapfusion, zhao2023mobilediffusion} facilitates the reverse process by optimizing the computations of UNets. 
~\cite{li2023snapfusion}~proposes 
an efficient UNet by identifying the redundancy of the original model and reducing the computation, while \cite{zhao2023mobilediffusion} further comprehensively analyzes and simplifies each component. These optimizations are orthogonal to other acceleration techniques.


{\bf Sampling step skipping}. DDIM~\cite{ddim:conf/iclr/SongME21} focuses on generalizing the Markovian diffusion of DDPM via a family of non-Markovian processes. These are deterministic and thus faster. Accordingly, it is able to reduce the required number of sampling steps on the trained DDPM model, without necessity of retraining. A noticeable limitation is that DDIM trades off the quality of generated data. 
as DDIM sampling approximates the procedure of the original model with skipped intermediate steps.

{\bf Knowledge distillation}. Previously explored for GANs~\cite{gandistil:conf/cvpr/Cazenavette00EZ23,gandistil:conf/icml/LiuLBL23}, distillation~\cite{distillation:conf/icml/AsadiDMAB23,distillation:conf/icml/CaoLHRWG23} allows to transfer knowledge from a large trained teacher model to a smaller student model for faster inference. To train a student model, progressive distillation~\cite{progress:salimans2022progressive} halves repeatedly the steps of a teacher model until the desired number of steps has been reached. 
~\cite{meng2023distillation} further extends the folding optimization to classifier-free guided diffusion implementation of Text-to-Image tasks. Although progressive distillation delivers increasingly efficient inference, each halving requires to train a new student model which multiples the training effort and impacts the output quality due to added approximation noise at each folding. Consistency models are proposed to improve the sample quality with few steps~\cite{DBLP:conf/icml/SongD0S23}. A consistency model can be directly trained or obtained by distilling a trained teacher model.

\section{Single-fold distillation}


We rethink knowledge distillation of 
diffusion models to reduce the number of sampling steps by proposing single-fold distillation. Instead of progressive multiple folds~\cite{progress:salimans2022progressive}, which introduces distortion and costs extra training effort at each fold, \alg directly distills the teacher model to a student model with a given number of steps in a single fold. One crucial challenge is the alignment from the teacher's to the student's hidden variables, as the Markov chain is defined by every two consecutive variables but the student has much fewer steps.

We first introduce the preliminaries of a DDPM model used as a teacher model, including the definition of the forward/reverse process and the training objective. Then we present \alg which defines the forward and reverse process of the student by aligning and matching the hidden variables of the teacher. Finally, we design the distillation algorithm for deployment, which shows the training procedure of the student accordingly.
For ease of presentation, we set the number of steps in the student model as a divisor of the number of steps in the teacher model, but the method is valid for an arbitrary number of student steps, i.e. fractional teacher/student step ratios.


\subsection{Preliminary}
\label{subsec:background}
DDPM is composed of a forward (noising) and a reverse (denoising) process, through $T$ steps. Its objective is to learn the  denoising process via a given forward process. 

\textbf{Reverse process:} The optimization objective of diffusion models~\cite{sohl2015deep} is derived by variational inference. Given the observed data $\boldsymbol{x}_0$, the diffusion model is a probabilistic model which specifies the joint distribution $p_\theta(\boldsymbol{x}_{0: T})$, where $\boldsymbol{x}_1, ..., \boldsymbol{x}_T$ are latent variables with the same dimensions as $\boldsymbol{x}_0$, and $\theta$ are learnable model parameters. $p_\theta(\boldsymbol{x}_{0: T})$ is 
a Markov chain that samples from $\boldsymbol{x}_T$ to $\boldsymbol{x}_0$,
$p_\theta\left(\boldsymbol{x}_{0: T}\right):=p_{\theta}\left(\boldsymbol{x}_T\right) \prod_{t=1}^T p_\theta\left(\boldsymbol{x}_{t-1} \mid \boldsymbol{x}_t\right).$
DDPM 
assumes that $p_{\theta}\left(\boldsymbol{x}_T\right)=\mathcal{N}\left(\boldsymbol{x}_T ; \mathbf{0}, \boldsymbol{I}\right)$ 
and
\begin{equation}
  p_\theta\left(\boldsymbol{x}_{t-1} \mid \boldsymbol{x}_t\right)=\mathcal{N}\left(\boldsymbol{x}_{t-1} ; \boldsymbol{\mu}_\theta\left(\boldsymbol{x}_t, t\right), \sigma_t^2 \boldsymbol{I}\right), \nonumber
\end{equation}
where $\sigma_t \in \mathbb{R}_{\geq 0}$. 
$p_\theta(\boldsymbol{x}_{0: T})$ is called the \textit{reverse process}. It gradually denoises a noise $\boldsymbol{x}_T \sim \mathcal{N}\left(\mathbf{0}, \boldsymbol{I}\right)$. 

\textbf{Forward process:} To derive a lower bound on the log likelihood of the observed data, diffusion models introduce the approximate posterior $q\left(\boldsymbol{x}_{1: T} \mid \boldsymbol{x}_0\right)$ which is a Markov chain that samples from $\boldsymbol{x}_1$ to $\boldsymbol{x}_T$, $q\left(\boldsymbol{x}_{1: T} \mid\boldsymbol{x}_0\right):=\prod_{t=1}^T q\left(\boldsymbol{x}_t \mid \boldsymbol{x}_{t-1}\right)$.
Then the log data likelihood can be decomposed as $\mathbb{E}\left[\log p_\theta\left(\boldsymbol{x}_0\right)\right] = \mathbb{E}_q\left[\log \frac{p_\theta\left(\boldsymbol{x}_{0: T}\right)}{q\left(\boldsymbol{x}_{1: T} \mid \boldsymbol{x}_0\right)}\right] + D_{\mathrm{KL}}\left(q\left(\boldsymbol{x}_{1:T} \mid \boldsymbol{x}_0\right) \| p_\theta\left(\boldsymbol{x}_{1:T} \mid \boldsymbol{x}_0 \right)\right)$. Thus, we have the lower bound $\mathbb{E}\left[\log p_\theta\left(\boldsymbol{x}_0\right)\right] \geq \mathbb{E}_q\left[\log \frac{p_\theta\left(\boldsymbol{x}_{0: T}\right)}{q\left(\boldsymbol{x}_{1: T} \mid \boldsymbol{x}_0\right)}\right]$. When training the diffusion model, to maximize $\mathbb{E}\left[\log p_\theta\left(\boldsymbol{x}_0\right)\right]$, the parameters $\theta$ are learned to minimize the negative evidence lower bound:
\begin{equation}
    \mathop{\arg\min}_{\theta} \mathbb{E}_{q}\left[\log q\left(\boldsymbol{x}_{1: T} \mid \boldsymbol{x}_0\right)-\log p_\theta\left(\boldsymbol{x}_{0: T}\right)\right]. 
    \label{eq:loss}
\end{equation}
The Markov chain $q\left(\boldsymbol{x}_{1: T} \mid \boldsymbol{x}_0\right)$ is called the \textit{forward process}. It progressively turns the data $x_0$ into noise. The conditional distribution in each forward step is defined as:
\begin{equation}
    q\left(\boldsymbol{x}_t \mid \boldsymbol{x}_{t-1}\right):=\mathcal{N}\left(\boldsymbol{x}_t; \sqrt{\frac{\alpha_t}{\alpha_{t-1}}} x_{t-1},\left(1-\frac{\alpha_t}{\alpha_{t-1}}\right) \boldsymbol{I}\right),
    \label{eq:qtt1}
\end{equation}
where $\alpha_t \in (0,1]$ and $\alpha_{1}, ..., \alpha_T$ is a decreasing sequence, which ensures that the values on the diagonal of the covariance matrix are positive. Reparameterizing using the definition in Eq.~(\ref{eq:qtt1}), an important property of the forward process is that:
\begin{equation}
    q\left(\boldsymbol{x}_t \mid \boldsymbol{x}_0\right) = \mathcal{N}\left(\boldsymbol{x}_t ; \sqrt{\alpha_t} \boldsymbol{x}_0,\left(1-\alpha_t\right) \boldsymbol{I}\right).
    \label{eq:tforwardprop}
\end{equation}

\textbf{Training and sampling:} According to the definition of the reverse $p$ and forward $q$ processes, $\alpha_t$ and $\sigma_t$ for all $t$ are not learnable parameters.
Thus, DDPM simplifies Eq.~(\ref{eq:loss}) as:
\begin{equation}
    \mathop{\arg\min}_{\theta} \sum_t \mathbb{E}_q [D_{\mathrm{KL}}\left(q\left(\boldsymbol{}{x}_{t-1} \mid \boldsymbol{x}_t, \boldsymbol{x}_0\right) \| p_\theta\left(\boldsymbol{x}_{t-1} \mid \boldsymbol{x}_t\right)\right)].
    \label{eq:loss_sim}
\end{equation}
DDPM further chooses the form of $\boldsymbol{\mu}_\theta\left(\boldsymbol{x}_t, t\right)$ to be:
\begin{equation}
    \boldsymbol{\mu}_\theta\left(\boldsymbol{x}_t, t\right) = \frac{1}{\sqrt{\frac{\alpha_t}{\alpha_{t-1}}}}\left(\boldsymbol{x}_t-\frac{1-\frac{\alpha_t}{\alpha_{t-1}}}{\sqrt{1-\alpha}_t} \boldsymbol{\epsilon}_\theta\left(\boldsymbol{x}_t, t\right)\right),
    \label{eq:pmean}
\end{equation}
where $\boldsymbol{\epsilon}_\theta$ is a function with trainable parameters $\theta$. According to the above definitions of the forward and reverse processes and applying the parameterization shown in Eq.~(\ref{eq:pmean}), Eq.~(\ref{eq:loss_sim}) can be simplified as $\mathop{\arg \min}_{\theta} L(\theta)$, where:
\begin{equation}
\footnotesize
    L(\theta):= \sum_{t=1}^{T} \mathbb{E}_{\boldsymbol{x}_0, \boldsymbol{\epsilon}_t}\left[\gamma_t\left\|\boldsymbol{\epsilon}_t-\boldsymbol{\epsilon}_\theta\left(\sqrt{\alpha_t} \boldsymbol{x}_0+\sqrt{1-\alpha_t} \boldsymbol{\epsilon}_t, t\right)\right\|^2\right]
    \label{eq:loss_teacher_f}
\end{equation}
with $\gamma_t = \frac{(\alpha_{t-1}-\alpha_t)^2}{2 \sigma_t^2 \alpha_t \alpha_{t-1}(1-\alpha_t)}$ and $\boldsymbol{\epsilon}_t \sim \mathcal{N}(\mathbf{0}, \boldsymbol{I})$. It is important to note that when deriving this loss, DDPM reparameterize $\boldsymbol{x}_t$ as
\begin{equation}
    \boldsymbol{x}_t = \sqrt{\alpha_t}\boldsymbol{x}_0 + \sqrt{1-\alpha_t} \boldsymbol{\epsilon}_t,
    \label{eq:repara_teacher}
\end{equation}
using the property in Eq.~(\ref{eq:tforwardprop}).
DDPM 
further simplifies 
$L$ by setting $\gamma_t = 1$ independent of $\alpha_{1:T}$.

The number of sampling steps $T$ decides the generalization capability and sampling cost of diffusion models. A big $T$ leads to a reverse process with high generalization capability that better captures the pattern of $x_0$. However, it also increases the sampling time and makes the sampling from DDPMs significantly slower than other generative models, e.g, GANs. Such inefficiency promotes our design of \alg to reduce the number of sampling steps via knowledge distillation.

\subsection{Single-fold Distilled Diffusion (\alg)}
\label{sec:methodology}

\begin{figure*}[t]
	\centering
	{       \includegraphics[width=13cm, height=3cm]{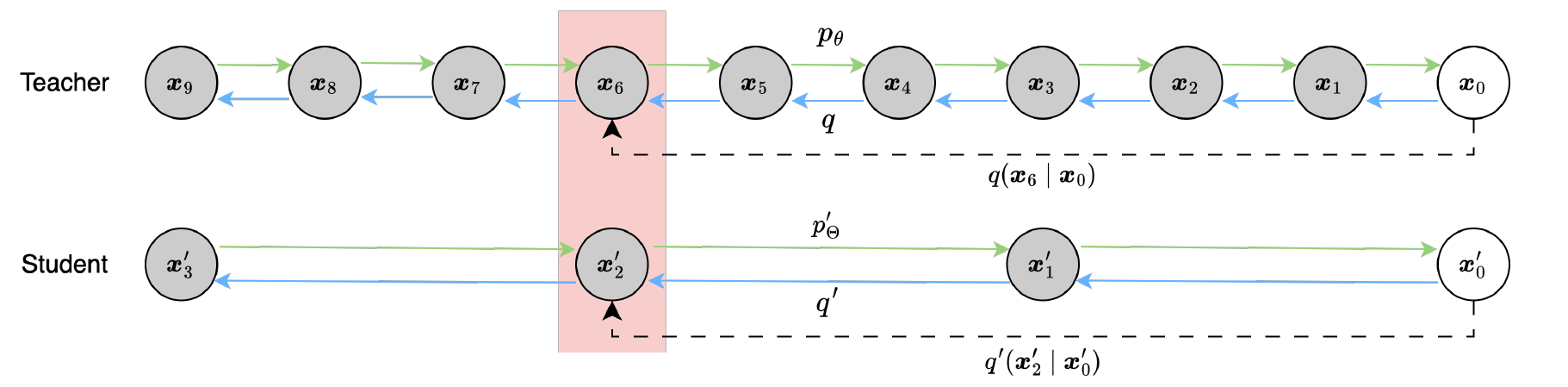} }
	\caption{Single-Fold Distillation of Diffusion Model (\alg). The student accelerates the inference by a small number of steps $T^{\prime}$ instead of a large $T$. We use $T=9$ and $T^{\prime}=3$ in the figure for readability.
    To align the teacher and student Markov chains, we propose to match the intermediate hidden variables to make, e.g., $q^{\prime}(\boldsymbol{x}^{\prime}_2 = \boldsymbol{x}_6 |\boldsymbol{x}^{\prime}_0 = \boldsymbol{x}_0$) equal to $q(\boldsymbol{x}_6|\boldsymbol{x}_0)$.
    }
    \label{fig:graph_model}
\end{figure*}


Instead of multiple folds like progressive distillation, which introduces distortion at every fold by retraining, we want to extract knowledge from the teacher model through a single fold. 
The first consideration is aligning steps from a $T$-step teacher to steps of a given smaller $T^{\prime}$-step student. Here, our goal is to distill the knowledge of the teacher model by mimicking its hidden variables ($\boldsymbol{x}_1,...,\boldsymbol{x}_T$) by a compressed student model\footnote{Hereon we use '$^{\prime}$' in symbols to indicate the corresponding student variables.} with hidden variables ($\boldsymbol{x}^{\prime}_1,...,\boldsymbol{x}^{\prime}_{T^{\prime}}$), where $T^{\prime} \ll T$.

A crucial challenge stems from the need to map a subset of multiple consecutive steps at the teacher into one single step at the student (see Fig.~\ref{fig:graph_model}).
For example, given an index\footnote{For simplicity, we assume that $T$ is divisible by $T^{\prime}$ but this is not necessary (see Sec.~\ref{subsec:flexible}).} $c \cdot t$, where $c = T/T^{\prime}$ and $c \in \mathbb{Z}$, according to Sec.~\ref{subsec:background}, the distributions we can explicitly obtain from the teacher are $q\left(\boldsymbol{x}_{c \cdot t} \mid \boldsymbol{x}_{c \cdot t-1}\right)$ (see Eq.~\ref{eq:qtt1}). However, mapping $\boldsymbol{x}^{\prime}_t$ with $\boldsymbol{x}_{c \cdot t}$ from the student to the teacher does not hold for the next step, i.e., $\boldsymbol{x}^{\prime}_{t-1}$ does not correspond to $\boldsymbol{x}_{c \cdot t-1}$, which is supposed to map $\boldsymbol{x}_{c(t-1)}$. Thus, when distilling the DDPM-like Markov chains, it is not reasonable to simply and straightforwardly simulate $q^{\prime}(\boldsymbol{x}^{\prime}_t|\boldsymbol{x}^{\prime}_{t-1})$ as $q(\boldsymbol{x}_{c \cdot t}|\boldsymbol{x}_{c \cdot t-1})$ while correspondingly estimating $p^{\prime}_{\Theta}(\boldsymbol{x}^{\prime}_t|\boldsymbol{x}^{\prime}_{t-1})$ as $p_{\theta}(\boldsymbol{x}_{c \cdot t}|\boldsymbol{x}_{c \cdot t-1})$, where the notations $q^{\prime}$ and $p^{\prime}_{\Theta}$ are used to represent the forward process and reverse process of the student model, respectively.

To overcome this challenge, we define a novel student model as follows. Note that the key definition for diffusion models is the forward process, as it determines the variational inference and dominates the training of the model. Therefore, to better distill the knowledge from the teacher, we design the forward process of the student $q^{\prime}\left(\boldsymbol{x}^{\prime}_{1: T^{\prime}} \mid \boldsymbol{x}^{\prime}_0\right)$ by extracting the teacher's forward process $q\left(\boldsymbol{x}_{1: T} \mid \boldsymbol{x}_0\right)$. Specifically, 
to ensure correspondence between the latent variables in the two diffusion models, given any $t \in [1, T^{\prime}]$, the proposed distillation algorithm \textbf{aims} to make $q^{\prime}(\boldsymbol{x}^{\prime}_t = \boldsymbol{x}_{c \cdot t} \mid \boldsymbol{x}^{\prime}_0 = \boldsymbol{x}_0)$ equal to $q(\boldsymbol{x}_{c \cdot t}|\boldsymbol{x}_0)$, which has a close-form solution (see Eq.~\ref{eq:tforwardprop}). After fixing $q^{\prime}$ by such an approximation in the forward process, the reverse process of the student is constructed accordingly. In the following, we show in detail how to define and train a student model.


\subsection{The forward process of the student model}
\label{subsec:forward}
To distill the DDPM teacher, we also assume that the student model has a Markovian forward process. Thus  $q^{\prime}\left(\boldsymbol{x}^{\prime}_{1: T^{\prime}} \mid \boldsymbol{x}^{\prime}_0\right)$ is a Markov chain that can be factorized as $$q^{\prime}\left(\boldsymbol{x}^{\prime}_{1: T^{\prime}} \mid \boldsymbol{x}^{\prime}_0\right):=\prod_{t=1}^{T^{\prime}} q^{\prime}\left(\boldsymbol{x}^{\prime}_t \mid \boldsymbol{x}^{\prime}_{t-1}\right).$$ 
As shown in Eq.~(\ref{eq:qtt1}), the forward process of the teacher is defined by a decreasing sequence $\alpha_{1}, ..., \alpha_T$. As for the student, different from the Gaussian distribution shown in Eq.~(\ref{eq:qtt1}), we set the student's forward process to the following form:
\begin{equation}
\footnotesize
q^{\prime}\left(\boldsymbol{x}^{\prime}_t \mid \boldsymbol{x}^{\prime}_{t-1}\right):=\mathcal{N}\left(\boldsymbol{x}^{\prime}_t; \sqrt{\frac{\alpha_{c \cdot t}}{\alpha_{c \cdot t-c}}} \boldsymbol{x}^{\prime}_{t-1},\left(1-\frac{\alpha_{c \cdot t}}{\alpha_{c \cdot t-c}}\right) \boldsymbol{I}\right),
    \label{eq:qtt1s}
\end{equation}
for all $t \in [1,T^{\prime}]$ based on the elements of the sequence $\{\alpha_t\}_{t=1}^{T}$ (the hyper-parameters of the given teacher). Then, according to this forward process definition of the student, we have the property:
\begin{equation}
q^{\prime}(\boldsymbol{x}^{\prime}_t \mid \boldsymbol{x}^{\prime}_0) =  \mathcal{N}\left(\boldsymbol{x}^{\prime}_t ; \sqrt{\alpha_{c \cdot t}} \boldsymbol{x}^{\prime}_0,\left(1-\alpha_{c \cdot t}\right) \boldsymbol{I}\right),
\label{eq:qxtx0}
\end{equation}
(see Appendix~\ref{app:sec1}). This  \textbf{ensures} that $q^{\prime}(\boldsymbol{x}^{\prime}_t = \boldsymbol{x}_{c \cdot t} \mid \boldsymbol{x}^{\prime}_0 = \boldsymbol{x}_0) = q(\boldsymbol{x}_{c \cdot t}|\boldsymbol{x}_0)$, where $\boldsymbol{x}^{\prime}_t$ corresponds to $\boldsymbol{x}_{c \cdot t}$. Thus, the student's Markov chain $q^{\prime}\left(\boldsymbol{x}^{\prime}_{1: T^{\prime}} \mid \boldsymbol{x}^{\prime}_0\right)$ can be regarded as a simplified copy of the teacher's forward process $q\left(\boldsymbol{x}_{1: T} \mid \boldsymbol{x}_0\right)$.

Before introducing the reverse process, we derive some important forward process posteriors. Applying Bayes' rule
    $q^{\prime}\left(\boldsymbol{x}^{\prime}_{t-1} \mid \boldsymbol{x}^{\prime}_t, \boldsymbol{x}^{\prime}_0\right)=q^{\prime}\left(\boldsymbol{x}^{\prime}_t \mid \boldsymbol{x}^{\prime}_{t-1}, \boldsymbol{x}^{\prime}_0\right) \frac{q^{\prime}\left(\boldsymbol{x}^{\prime}_{t-1} \mid \boldsymbol{x}^{\prime}_0\right)}{q^{\prime}\left(\boldsymbol{x}^{\prime}_t \mid \boldsymbol{x}^{\prime}_0\right)}$,
and the Markov chain property that 
   $q^{\prime}\left(\boldsymbol{x}^{\prime}_t \mid \boldsymbol{x}^{\prime}_{t-1}, \boldsymbol{x}^{\prime}_0\right) = q^{\prime}\left(\boldsymbol{x}^{\prime}_t \mid \boldsymbol{x}^{\prime}_{t-1}\right)$,
we have the 
posteriors:
\begin{equation}
\footnotesize
\begin{aligned}
    q^{\prime}\left(\boldsymbol{x}^{\prime}_{t-1} \mid \boldsymbol{x}^{\prime}_t, \boldsymbol{x}^{\prime}_0\right) =   
    \mathcal{N}\left(\boldsymbol{x}^{\prime}_{t-1} ; \frac{(1-\alpha_{c \cdot t-c})\sqrt{\alpha_{c \cdot t}}}{(1-\alpha_{c \cdot t})\sqrt{\alpha_{c \cdot t-c}}}\boldsymbol{x}^{\prime}_t + \frac{\alpha_{c \cdot t-c}-\alpha_{c \cdot t}}{(1-\alpha_{c \cdot t})\sqrt{\alpha_{c \cdot t-c}}}\boldsymbol{x}^{\prime}_0, \sigma^{\prime}_t \boldsymbol{I}\right),
\end{aligned}
\label{eq:qposter}
\end{equation}
where $\sigma^{\prime}_t = \frac{(1-\alpha_{c \cdot t-c})(\alpha_{c \cdot t-c}-\alpha_{c \cdot t})}{(1-\alpha_{c \cdot t})\alpha_{c \cdot t-c}}$ (see Appendix~\ref{app:sec2} for the derivation).



\subsection{The reverse process of the student model}

In the following, we define the reverse process of the student model according to its forward process. Similarly, it is also a Markov chain represented as  $p^{\prime}_\Theta(\boldsymbol{x}^{\prime}_{0: T^{\prime}}) := p^{\prime}_{\Theta}\left(\boldsymbol{x}^{\prime}_T\right) \prod_{t=1}^{T^{\prime}} p^{\prime}_\Theta\left(\boldsymbol{x}^{\prime}_{t-1} \mid \boldsymbol{x}^{\prime}_t\right),$ where $p^{\prime}_{\Theta}\left(\boldsymbol{x}^{\prime}_T\right) = \mathcal{N}\left( \mathbf{0}, \boldsymbol{I}\right)$ and $\Theta$ represents the learnable parameters of the student. 

Then, we decide the form of $p^{\prime}_\Theta\left(\boldsymbol{x}^{\prime}_{t-1} \mid \boldsymbol{x}^{\prime}_t\right)$. Referring to Eq.~(\ref{eq:qxtx0}), we can reparameterize $\boldsymbol{x}^{\prime}_t$ as a linear combination of $\boldsymbol{x}^{\prime}_0$ and $\epsilon_t \sim \mathcal{N}(\mathbf{0}, \boldsymbol{I})$ that is $\boldsymbol{x}^{\prime}_t=\sqrt{\alpha_{c\cdot t}} \boldsymbol{x}^{\prime}_0+\sqrt{1-\alpha_{c\cdot t}} \epsilon_t$. Let the model $\epsilon^{\prime}_{\Theta}(\boldsymbol{x}^{\prime}_t, t)$ predict $\epsilon_t$, we have a prediction of $\boldsymbol{x}^{\prime}_0$ given $\boldsymbol{x}^{\prime}_t$:
\begin{equation}
f_\Theta^{(t)}\left(\boldsymbol{x}^{\prime}_t\right):=\left(\boldsymbol{x}^{\prime}_t-\sqrt{1-\alpha_{c \cdot t}} \cdot \epsilon^{\prime}_\Theta\left(\boldsymbol{x}^{\prime}_t,t\right)\right) / \sqrt{\alpha_{c\cdot t}}
\label{eq:predictor}
\end{equation}
According to Eq.~(\ref{eq:loss_sim}), we know that in the student diffusion model, $p^{\prime}_\Theta\left(\boldsymbol{x}^{\prime}_{t-1} \mid \boldsymbol{x}^{\prime}_t\right)$ is a distribution that predicts $q^{\prime}(\boldsymbol{x}^{\prime}_{t-1} \mid \boldsymbol{x}^{\prime}_t, \boldsymbol{x}^{\prime}_0)$. Thus, we define that for all $t \in [1, T^{\prime}]$,
\begin{equation}
    p^{\prime}_\Theta\left(\boldsymbol{x}^{\prime}_{t-1} \mid \boldsymbol{x}^{\prime}_t\right) = q^{\prime}(\boldsymbol{x}^{\prime}_{t-1} \mid \boldsymbol{x}^{\prime}_t, f_\Theta^{(t)}\left(\boldsymbol{x}^{\prime}_t\right)).
    \label{eq:pdef}
\end{equation}
Then, based on Eq.~(\ref{eq:qposter}) and Eq.~(\ref{eq:pdef}), by replacing $\boldsymbol{x}^{\prime}_0$ as the predictor $f_\Theta^{(t)}\left(\boldsymbol{x}^{\prime}_t\right)$, we have: 
\begin{equation}
     p^{\prime}_\Theta\left(\boldsymbol{x}^{\prime}_{t-1} \mid \boldsymbol{x}^{\prime}_t\right) =  \mathcal{N}\left(\boldsymbol{x}^{\prime}_{t-1} ; \boldsymbol{\mu}^{\prime}_\Theta\left(\boldsymbol{x}^{\prime}_t, t\right), \sigma^{\prime}_t \boldsymbol{I}\right),
     \label{eq:pt1t}
\end{equation}
where the mean is:
\begin{equation}
\footnotesize
\boldsymbol{\mu}^{\prime}_\Theta\left(\boldsymbol{x}^{\prime}_t, t\right) = \frac{(1-\alpha_{ct-c})\sqrt{\alpha_{ct}}}{(1-\alpha_{ct})\sqrt{\alpha_{ct-c}}}\boldsymbol{x}^{\prime}_t + \frac{\alpha_{ct-c}-\alpha_{ct}}{(1-\alpha_{ct})\sqrt{\alpha_{ct-c}}}f_\Theta^{(t)}\left(\boldsymbol{x}^{\prime}_t\right).
\label{eq:revers_mean}
\end{equation}

\subsection{Distillation procedure}


Having defined the forward and reverse processes, here we design the algorithm for training the student model. The reparameterization of  $\boldsymbol{\mu}^{\prime}_\Theta\left(\boldsymbol{x}^{\prime}_t, t\right)$ shown in Eq.~(\ref{eq:revers_mean}) can be applied to derive the training loss of the student. For maximizing the log data likelihood of the observed data $\{\boldsymbol{x}^{\prime}_0\}$ on the student\footnote{In distillation, student and teacher observe the same training samples. Thus, $\boldsymbol{x}^{\prime}_0$ always equals $\boldsymbol{x}_0$ and $\{x^{\prime}_0\}$ is equivalent to $\{x_0\}$. To avoid confusion, we use $\{x^{\prime}_0\}$ to represent the observed data when training the student.}, referring to Eq.~(\ref{eq:loss_sim}), it is equivalent to minimize the Kullback-Leibler divergence between Eq.~(\ref{eq:qposter}) and Eq.~(\ref{eq:pt1t}). Then by using Eq.~(\ref{eq:qposter}), (\ref{eq:pt1t}), (\ref{eq:revers_mean}) and (\ref{eq:predictor}), we have the following loss for training the student (see Appendix~\ref{app:sec3} for the details):
\begin{equation}
\footnotesize
    L(\Theta):= \sum_{t=1}^{T} \mathbb{E}_{\boldsymbol{x}^{\prime}_0, \boldsymbol{\epsilon}^{\prime}_t}\left[\gamma^{\prime}_t\left\|\boldsymbol{\epsilon}^{\prime}_t-\boldsymbol{\epsilon}_\Theta \left(\sqrt{\alpha_{c\cdot t}} \boldsymbol{x}^{\prime}_0+\sqrt{1-\alpha_{c \cdot t}} \boldsymbol{\epsilon}^{\prime}_t, t\right)\right\|^2\right],
    \label{eq:loss_student_initial}
\end{equation}
where $\gamma^{\prime}_t = \frac{(\alpha_{c\cdot t-c}-\alpha_{c\cdot t})^2}{2 {\sigma^{\prime}_t}^2 \alpha_{c\cdot t} \alpha_{c\cdot t-c}(1-\alpha_{c\cdot t})}$. By the loss (Eq.~\ref{eq:loss_student_initial}), we can train the defined student model from scratch using the observed data $\{\boldsymbol{x}^{\prime}_0\}$. However, in order to extract the knowledge from the trained teacher, in the following we connect the training of the student with the trained teacher.

In the derivation of the loss $L(\Theta)$ (Eq.~\ref{eq:loss_student_initial}), we use the following reparameterization:
\begin{equation}
    \boldsymbol{x}^{\prime}_t = \sqrt{\alpha_{c\cdot t}} \boldsymbol{x}^{\prime}_0+\sqrt{1-\alpha_{c \cdot t}} \boldsymbol{\epsilon}^{\prime}_t.
    \label{eq:repara_student}
\end{equation}
According to Eq.~(\ref{eq:repara_teacher}), we know that $\boldsymbol{x}_{c\cdot t} = \sqrt{\alpha_{c\cdot t}}\boldsymbol{x}_0 + \sqrt{1-\alpha_{c\cdot t}} \boldsymbol{\epsilon}_{c\cdot t}$. In our distillation, we want to make the hidden variable $\boldsymbol{x}^{\prime}_t$ of the student equal to its corresponding hidden variable $\boldsymbol{x}_{c\cdot t}$ of the teacher. As the student and the teacher use the same observed data, $\boldsymbol{x}^{\prime}_0 = \boldsymbol{x}_0$, by Eq.~(\ref{eq:repara_student}), if we assume $\boldsymbol{\epsilon}^{\prime}_t = \boldsymbol{\epsilon}_{c\cdot t}$, then we can have $\boldsymbol{x}^{\prime}_t = \boldsymbol{x}_{c\cdot t}$. By the training loss of the teacher $L(\theta)$ (Eq.~\ref{eq:loss_teacher_f}), we know that the output $\boldsymbol{\epsilon}_\theta\left(\sqrt{\alpha_{c\cdot t}} \boldsymbol{x}_0+\sqrt{1-\alpha_{c\cdot t}} \boldsymbol{\epsilon}_{c\cdot t}, c\cdot t\right)$ from the trained function $\boldsymbol{\epsilon}_{\theta}$ of the teacher is a good predictor for $\boldsymbol{\epsilon}_{c\cdot t}$ (also for $\boldsymbol{\epsilon}^{\prime}_t$). Thus, we naturally rewrite the student loss (Eq.~\ref{eq:loss_student_initial}) as 
\begin{equation}
\footnotesize 
    L(\Theta):= \sum_{t=1}^{T} \mathbb{E}_{\boldsymbol{x}^{\prime}_0, \boldsymbol{\epsilon}^{\prime}_t}[\gamma^{\prime}_t\|\boldsymbol{\epsilon}_\theta\left(\sqrt{\alpha_{c\cdot t}} \boldsymbol{x}^{\prime}_0+\sqrt{1-\alpha_{c\cdot t}} \boldsymbol{\epsilon}^{\prime}_{t}, c\cdot t\right) 
    -\boldsymbol{\epsilon}_\Theta \left(\sqrt{\alpha_{c\cdot t}} \boldsymbol{x}^{\prime}_0+\sqrt{1-\alpha_{c \cdot t}} \boldsymbol{\epsilon}^{\prime}_t, t\right)\|^2].
    \label{eq:loss_student_final}
\end{equation}

\begin{minipage}{0.55\textwidth}
\vspace{-2em}
\begin{algorithm}[H]
   \caption{The distillation procedure of \alg}
   \label{alg:training}
\begin{algorithmic}[1]
   \STATE {\bfseries Input:} dataset $D = \{x^{\prime}_0\}$ and teacher $\boldsymbol{\epsilon}_{\theta}$ 
   \REPEAT
   \STATE $\boldsymbol{x}^{\prime}_0 \sim D$.
   \STATE $t \sim {Uniform}(\{1,...,T^{\prime}\})$
   \STATE $\boldsymbol{\epsilon} \sim \mathcal{N}(\mathbf{0}, \boldsymbol{I})$
   \STATE $\hat{\boldsymbol{\epsilon}} = \boldsymbol{\epsilon}_\theta\left(\sqrt{\alpha_{c\cdot t}} \boldsymbol{x}^{\prime}_0+\sqrt{1-\alpha_{c\cdot t}} \boldsymbol{\epsilon}, c\cdot t\right)$
   \STATE $L_{\Theta} = \left\|\hat{\boldsymbol{\epsilon}} -\boldsymbol{\epsilon}_\Theta \left(\sqrt{\alpha_{c\cdot t}} \boldsymbol{x}^{\prime}_0+\sqrt{1-\alpha_{c \cdot t}} \boldsymbol{\epsilon}, t\right)\right\|^2$ 
   \STATE Updating the student by $\nabla_{\Theta} L_{\Theta}$
   \UNTIL{converged}
\end{algorithmic}
\end{algorithm}
\end{minipage}
\hfill
\begin{minipage}{0.4\textwidth}
\vspace{-2.8em}
\begin{algorithm}[H]
   \caption{Sampling of \alg}
   \label{alg:sampling}
\begin{algorithmic}[1]
   \STATE $\boldsymbol{x}^{\prime}_{T^{\prime}} \sim \mathcal{N}(\mathbf{0}, \boldsymbol{I})$
   \FOR{$t = T^{\prime},...,1$}
   \IF{t > 1}
   \STATE $\boldsymbol{\epsilon} \sim \mathcal{N}(\mathbf{0}, \boldsymbol{I})$
   \ELSE
   \STATE $\boldsymbol{\epsilon} = 0$
   \ENDIF
   \STATE $\boldsymbol{x}^{\prime}_{t-1} = \boldsymbol{\mu}^{\prime}_\Theta\left(\boldsymbol{x}^{\prime}_{t}, t\right) + \sigma^{\prime}_{t} \boldsymbol{\epsilon}$
   \ENDFOR
\end{algorithmic}
\end{algorithm}
\end{minipage}

\textbf{Training:} We train the student according to the loss (Eq.~\ref{eq:loss_student_final}). Algorithm~\ref{alg:training} illustrates the training procedure of the student. During implementation in Sec.~\ref{sec:experiments}, we simplify the loss by setting $\gamma^{\prime}_t=1$, a simpler approach shown beneficial for sample quality. 

\textbf{Sampling:} Since the reverse process of the student is defined by Eq.~(\ref{eq:pt1t}), given the input $\boldsymbol{x}^{\prime}_{T^{\prime}}$, we can steadily sample all variables from $\boldsymbol{x}^{\prime}_{T^{\prime}-1}$ to $\boldsymbol{x}^{\prime}_0$. Note that the student model is also a DDPM model. Thus, on the distilled student, we can apply any sampling algorithm compatible with DDPM, e.g., DDIM, ODE solvers \cite{lu2022dpm}, etc.


\subsection{Distillation on flexible sub-sequence}
\label{subsec:flexible}

In the previous derivation, the student extracts the knowledge from a special variable subset $\{\boldsymbol{x}_{c\cdot t}\}$ of the teacher where $t \in [1, T^{\prime}]$. Indeed, our proposed method can be extended to a more general case where we distill the knowledge from any given subset $\{\boldsymbol{x}_{\phi_0},...,\boldsymbol{x}_{\phi_{T^{\prime}}}\}$ of the teacher where $\phi$ is an increasing sub-sequence of $\{0,...,T\}$, $\phi_0 = 0$ and $\phi_{T^{\prime}} = T$. In this general case, the Gaussian mean of $p^{\prime}_\Theta\left(\boldsymbol{x}^{\prime}_{t-1} \mid \boldsymbol{x}^{\prime}_t\right)$ (see Eq.~\ref{eq:pt1t}) is 
\begin{equation*}
\footnotesize
\boldsymbol{\mu}^{\prime}_\Theta\left(\boldsymbol{x}^{\prime}_t, t\right) = \frac{(1-\alpha_{\phi_{t-1}})\sqrt{\alpha_{\phi_{t}}}}{(1-\alpha_{\phi_{t}})\sqrt{\alpha_{\phi_{t-1}}}}\boldsymbol{x}^{\prime}_t + \frac{\alpha_{\phi_{t-1}}-\alpha_{\phi_{t}}}{(1-\alpha_{\phi_{t}})\sqrt{\alpha_{\phi_{t-1}}}}\mathcal{F}_\Theta^{(t)}\left(\boldsymbol{x}^{\prime}_t\right),
\end{equation*}
where $\mathcal{F}_\Theta^{(t)}\left(\boldsymbol{x}^{\prime}_t\right):=\left(\boldsymbol{x}^{\prime}_t-\sqrt{1-\alpha_{\phi_{t}}} \cdot \epsilon^{\prime}_\Theta\left(\boldsymbol{x}^{\prime}_t,t\right)\right) / \sqrt{\alpha_{\phi_t}}$ is a prediction of $\boldsymbol{x}^{\prime}_0$ given $\boldsymbol{x}^{\prime}_t$. Remarkably, $T$ no longer needs to be divisible by $T^{\prime}$. 
More details of this flexible sub-sequence setting is shown in Appendix~\ref{app:sec4}. 
This extension is beneficial as it greatly expands the applicability of our distillation under various steps of teacher diffusion models.






\section{Evaluation}
\label{sec:experiments}

\begin{figure}[th!]
\vspace{-2em}
	\centering
 \subfloat[Cifar-10]{	\includegraphics[width=0.25\linewidth]{"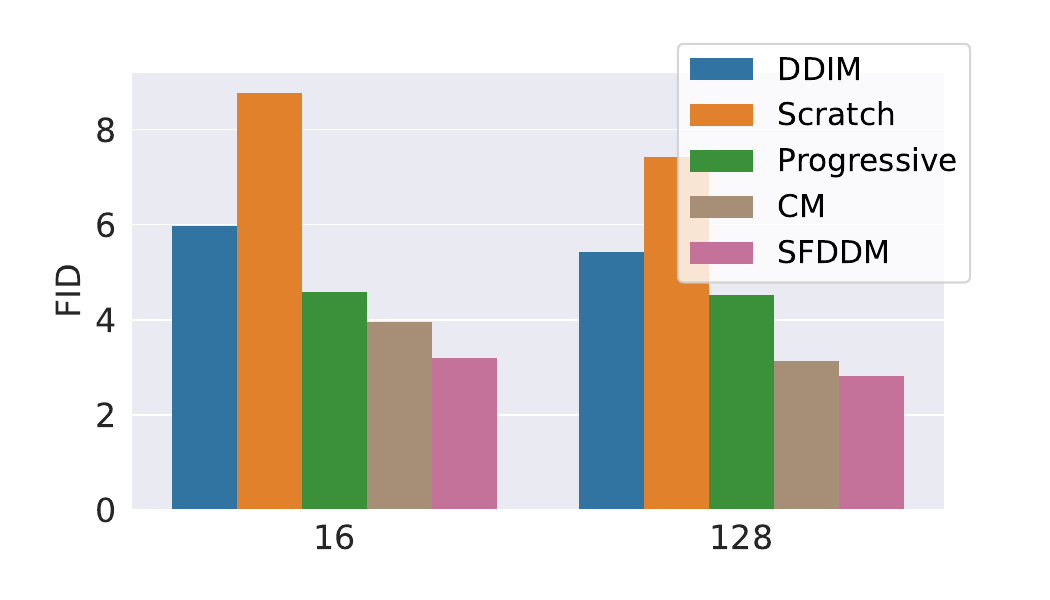"}
 }
 \subfloat[Bedroom]{	\includegraphics[width=0.25\linewidth]{"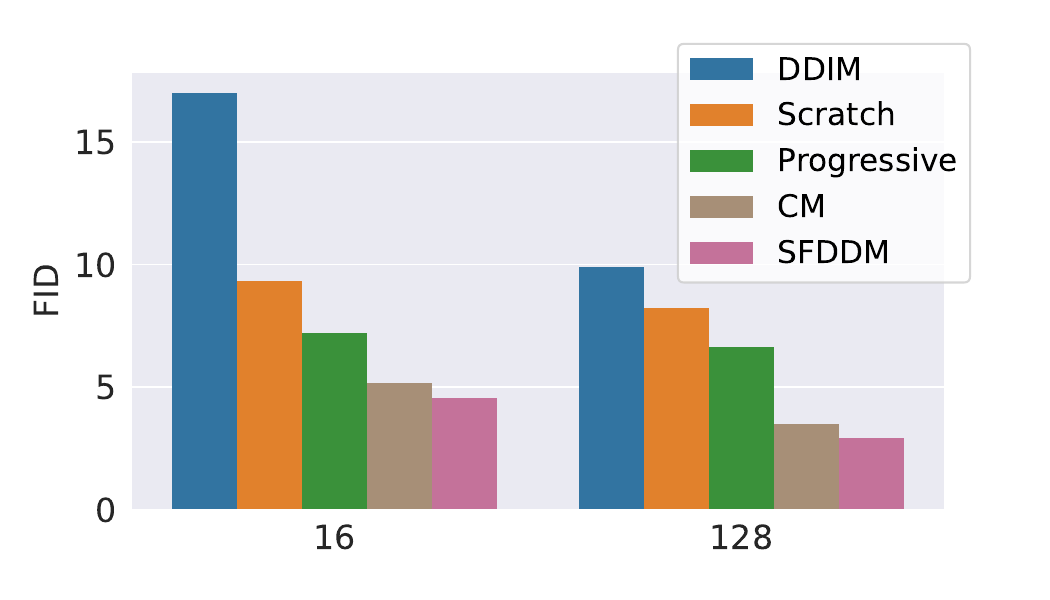"}
 }
 \subfloat[Church]{	\includegraphics[width=0.25\linewidth]{"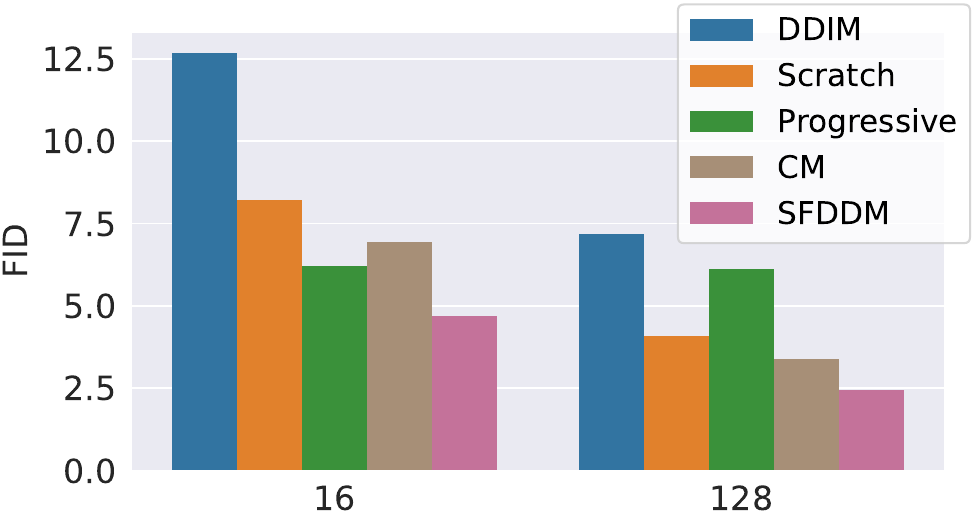"}
 }
 \subfloat[CelebA]{	\includegraphics[width=0.24\linewidth]{"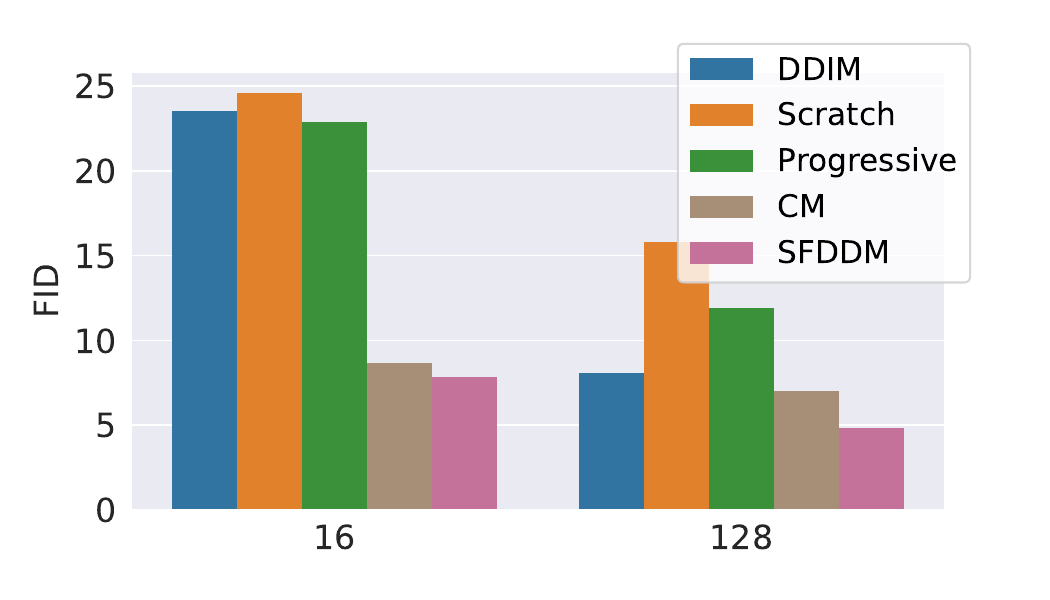"}
 }
 \vspace{-0.5em}
 \caption{FID under different number of sampling steps from the teacher $T=1024$, on four datasets.}
  \label{fig:fid_quality}
 \end{figure}

\begin{wrapfigure}{r}{0.4\textwidth}
\vspace{-1em}
  \begin{center}
    \includegraphics[width=0.4\textwidth]{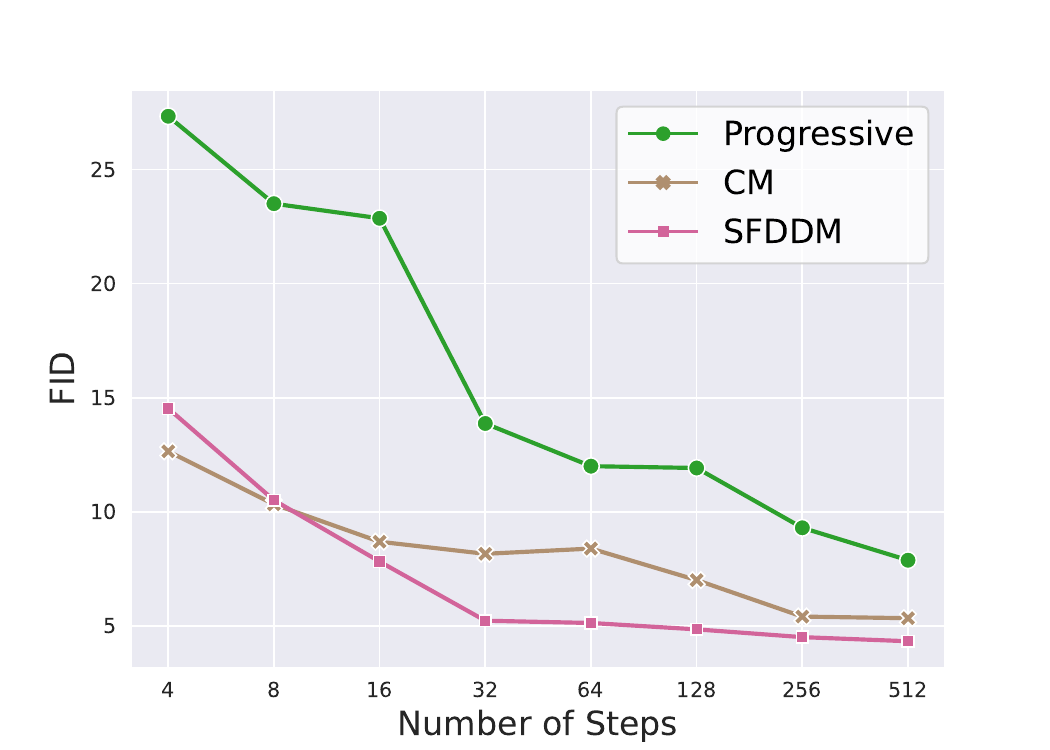}
  \end{center}
  \vspace{-1em}
  \caption{FID of the methods with different number of sampling steps on CelebA-HQ.}
  \label{fig:different_steps}
\end{wrapfigure}



\alg distills the knowledge of any arbitrary DDPM-like teacher models with efficiency and high sampling quality. We demonstrate its effectiveness by four image benchmark datasets: CIFAR-10~\cite{cifar:krizhevsky2009learning}, CelebA~\cite{celeba:liu2015faceattributes}, LSUN-Church and LSUN-Bedroom~\cite{lsun:yu15lsun}. For each dataset, we distill the same teacher diffusion model of DDPM into 16, 100, or 128 student steps. Note that although the original DDPM contains 1000 sampling steps, in this paper, we set it as 1024 steps in order to compare with progressive distillation ~\cite{progress:salimans2022progressive}, which requires a number of steps that is a power of 2 for progressive halving. The evaluation metric applied is FID together with perceptual visualization of sampled images. Besides image benchmarks, we also evaluate \alg on a tabular 2D Swiss Roll dataset (see Appendix~\ref{app_sec:tabular}). All remaining details on our experiments are given in Appendix~\ref{app:exp}.

\subsection{Sampling quality and efficiency}

\begin{figure*}[h]
\vspace{-1em}
	\centering
	{
            \includegraphics[width=1\textwidth]{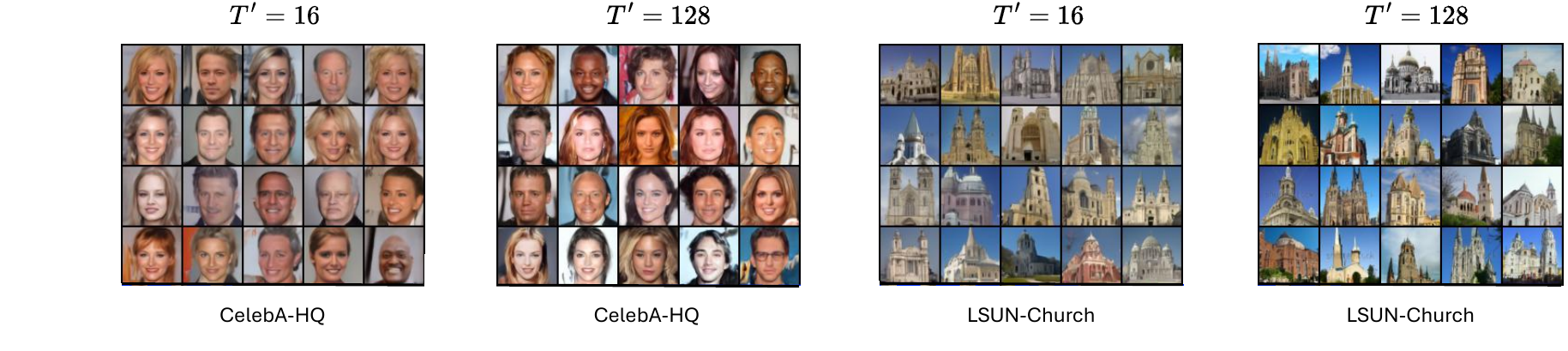} }
	\caption{Generated samples from \alg on different $T^{\prime}$.}
        \label{fig:quality_comp}
\vspace{-0.5em}
\end{figure*}

To demonstrate the quality and efficiency of \alg, we report the FID in Fig.~\ref{fig:fid_quality} 
on all four datasets comparing against DDIM, progressive distillation (``Progressive''), Consistency model (CM)~\cite{DBLP:conf/icml/SongD0S23}, and training directly on the smaller model with the same dataset as the teacher model (``From scratch''). Fig.~\ref{fig:quality_comp} also shows the images sampled by \alg
with both $T^{\prime} = 16$ and $T^{\prime} = 128$ student steps.

In general, our \alg achieves the best FID over different datasets and different small $T^{\prime}$. From the results, we observe that the sampling data quality increases with the increase of $T^{\prime}$, as more sampling steps match more hidden variables of the teacher model, better approximating the teacher distributions. This can also be clearly noticed from Fig.~\ref{fig:quality_comp} by generating images with more fine-grained features, e.g., hair texture, and details on clothes. 

Comparing the baselines, DDIM achieves mostly the lowest quality in sampled 
images as shown by high FID scores.
This stems from the nature that DDIM \textbf{does not} retrain a student model but focuses on improving the inference efficiency of the original model. Thus, it benefits from simplicity but falls short in terms of quality, when skipping too many intermediate steps during sampling, e.g., $T^{\prime} = 16$. On the other hand, training from scratch on a small diffusion model comes in as the second worst in terms of FID, due to the difficulty of capturing image features with a small number of steps. Progressive distillation yields marginal improvement as multiple folding and retraining increases distortion from teacher knowledge due to noises it introduces at each folding. An interesting finding emerges in the context of the LSUN-Church dataset with 128 sampling steps. Here, training from scratch surpasses progressive distillation in FID performance. This can be attributed to the fact that, with a relatively large number of steps, direct training exhibits superior quality compared to progressive distillation, where systematic distortion occurs fold by fold. This is consistent with the conclusion in~\cite{progress:salimans2022progressive}.
We also compare the student FID for progressive distillation, CM, and \alg with varying sampling steps from 4 to 512 in Fig.~\ref{fig:different_steps}. Overall, \alg outperforms the baselines with the sole exception of CM when $T=4$. This is within our expectations since CM is intrinsically designed for small values of $T$ but of different model types, rather than 
aiming for a high similarity 
to the teacher DDPM.






\subsection{Distillation with different sub-sequences}

In accordance with Sec.~\ref{subsec:flexible}, our algorithm can be extended to accommodate flexible sub-sequences of the student model. Here, we compare FID values for different choices to demonstrate their impact. 
Specifically, the different cases of flexible sub-sequence are designed by various degrees of concentration of mapped steps around the midpoint element. 
We define it by the percentage of elements distributed uniformly within a 5\% range near the midpoint element (i.e., 512) while the others are uniformly distributed in the remaining range of steps. Moreover, 
In our results presented in Tab.~\ref{tab:flexibility}, we include the concentration degrees of 40\%, 20\%, plus Scattered.  
``Scattered'' refers to the student model matching a sparse sub-sequence spread across the full teacher Markov chain. 
Scattered allows to cover more knowledge of the noising/denoising procedure form the teacher. Yet, concentrated choices, in which elements are nearby and centered in a partial part of the teacher chain, are still able to distill high-quality diffusion models, as shown by similarity in  FID scores.

\begin{table}[h]
\centering
  \caption{Distilling the same teacher with different sub-sequences on CelebA-HQ with $T^{\prime} = 16$.}
  \label{tab:flexibility}
  \begin{tabular}{c|c|c|c}
    \toprule
    sub-sequence & Concentrated (40\%) &  Concentrated (20\%) & Scattered\\
    \midrule
    FID & 7.95 & 7.59 & 7.82 \\

  \bottomrule
\end{tabular}
\end{table}


\subsection{Consistency between teacher and student}

\begin{figure*}
	\centering
	{
            \includegraphics[width=0.9\textwidth]{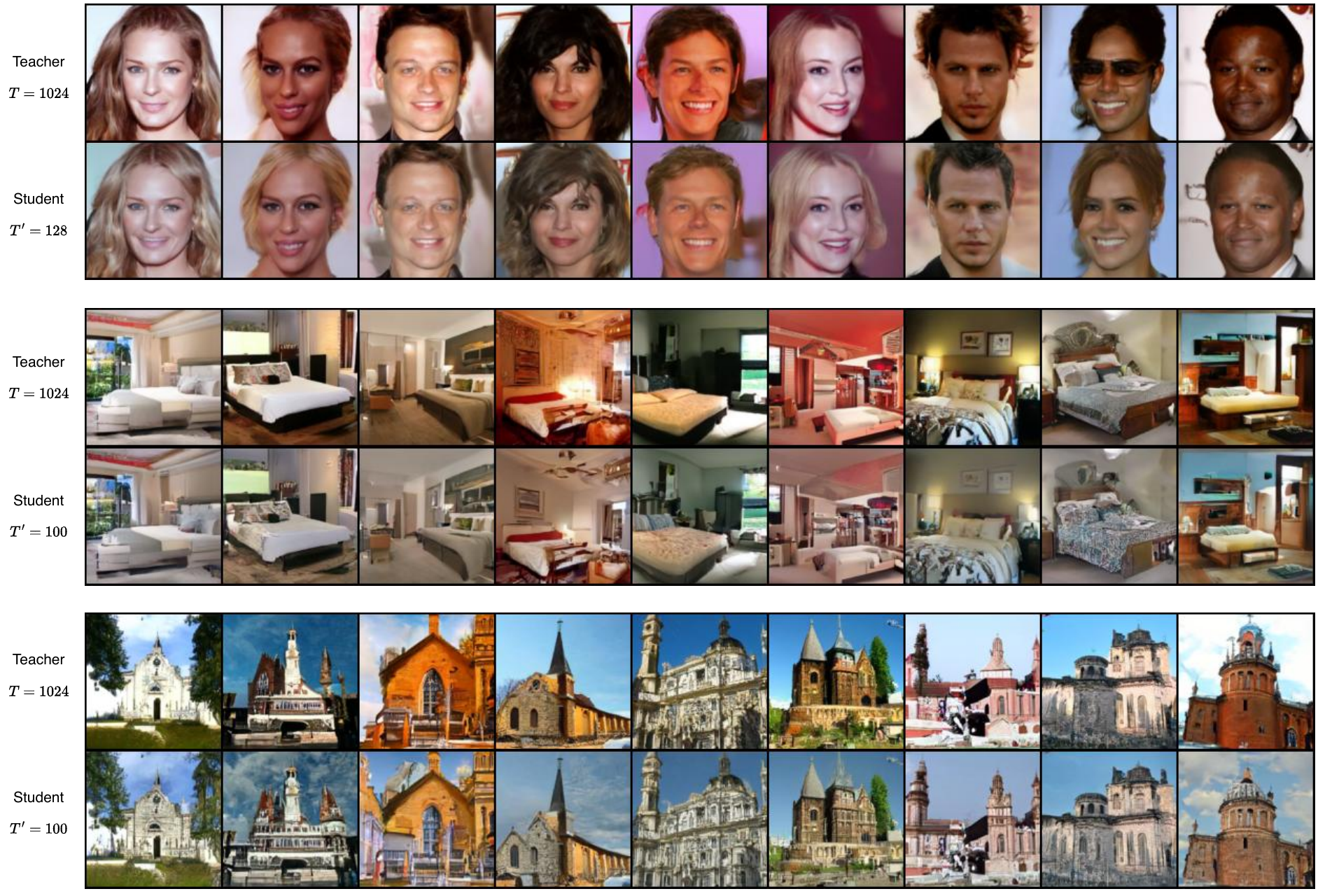} }
	\caption{Consistency on CelebA-HQ, LSUN-Bedroom and LSUN-Church: inputing the same noise.} 
        \label{fig:consistency}
\end{figure*}

We zoom into the consistency between images sampled by the teacher and student models when inputting identical noise. It is important to note that for a fair comparison of the distillation results, we use the DDIM sampling method (also applied in Sec.~\ref{subsec:interpolation}) for both the DDPM teacher and our \alg student. This is because the output of the original DDPM sampling is not solely determined by the input noise, owing to the introduced random factor during the stochastic generative process. In contrast, DDIM sampling ensures pair correspondence, i.e., same input, same output, facilitating a clear and accurate comparison. The outcomes across three distinct datasets are illustrated in Fig.~\ref{fig:consistency}. It is noteworthy that we employ different sampling steps for the student among different datasets, thereby validating the flexibility of our approach under varying numbers of sub-sequences, where the number of steps is not a power of 2. As evidenced by the images, it is clearly observed that inputting identical noise leads to similar outputs, showing our effectiveness in transferring knowledge from the teacher to a student having only approximately 1/10 of the steps.

\subsection{Interpolation on the teacher and the student}
\label{subsec:interpolation}

\begin{figure*}
	\centering
	{
    \includegraphics[width=0.9\textwidth]{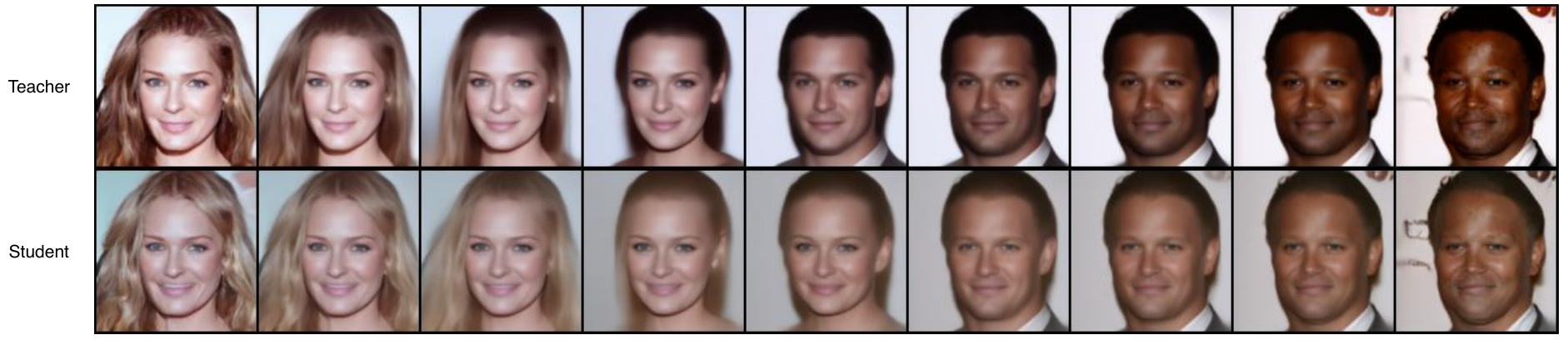} }
	\caption{Interpolation on the distilled student and original teacher.}
        \label{fig:interpotation}
\end{figure*}

We further assess the efficacy of knowledge transfer through measuring the similarity on semantic interpolation between the teacher and student models. We evenly interpolate between two given noises to showcase the intermediate sampled data in Fig.~\ref{fig:interpotation}. The results reveal a stable visual interpolation in the generated images since the input noise encodes distinctive high-level features of the image. Consequently, the interpolation data implicitly captures the features between the two noises into perceptually similar outputs. When comparing the output of the teacher and student, we observe similarity in each interpolation, showcasing the effectiveness of distillation as the student successfully inherits a stable and consistent sampling capability from the teacher. 

\section{Limitations} 

\label{sec:limitation}

The limitation of \alg is that it is not intrinsically specialized for extremely few sampling steps e.g., 1 or 2 steps. Its performance on very few sampling steps is not as good as CM. The reason is that \alg maintains the same model type as the teacher while CM forgoes such a property. An \alg student Markov chain is a simplified copy of the teacher's Markov chain. This maintains not only the model type but also provides common behaviors, e.g., leading to consistent outputs given the same input noise.

\section{Conclusion}

In this paper, we propose a novel and effective single-fold distillation method for diffusion models, \alg. In contrast to the prior study of progressive distillation, \alg is able to compress any $T$-step teacher model into any $T'$-step student model in single-fold distillation. The key enabling features are (i) new derivation of the forwarding process of the student model, which leverages the reparameterization of the teacher model and approximation of their Markovian states; and (ii) optimization of the denoise process of the student model by minimizing the difference of model outputs and distribution of the hidden variables. Our evaluation results on four datasets show that \alg achieves remarkable quality on FID allowing to strike better quality-compute tradeoffs.

\bibliographystyle{plain}
\bibliography{example_paper}

\newpage
\appendix

\onecolumn
\section{Proofs and extended derivations}

\subsection{The property of the forward process of the student}
\label{app:sec1}
\begin{lemma}
\label{lem:usefullemma}
For the Markovian assumption on the forward process $q^{\prime}\left(\boldsymbol{x}^{\prime}_{1: T^{\prime}} \mid \boldsymbol{x}^{\prime}_0\right):=\prod_{t=1}^{T^{\prime}} q^{\prime}\left(\boldsymbol{x}^{\prime}_t \mid \boldsymbol{x}^{\prime}_{t-1}\right)$ of the student and $q^{\prime}\left(\boldsymbol{x}^{\prime}_t \mid \boldsymbol{x}^{\prime}_{t-1}\right)$ defined in Eq.~(\ref{eq:qtt1s}), we have
\begin{equation}
    q^{\prime}(\boldsymbol{x}^{\prime}_t \mid \boldsymbol{x}^{\prime}_0) =  \mathcal{N}\left(\boldsymbol{x}^{\prime}_t ; \sqrt{\alpha_{c \cdot t}} \boldsymbol{x}^{\prime}_0,\left(1-\alpha_{c \cdot t}\right) \boldsymbol{I}\right),
\end{equation}
\end{lemma}

\begin{proof} 
According to Eq.~(\ref{eq:qtt1s}), $\boldsymbol{x}^{\prime}_t$ can be reparameterized as
\begin{equation}
    \begin{aligned}
        \boldsymbol{x}^{\prime}_{t} &= \sqrt{\frac{\alpha_{c \cdot t}}{\alpha_{c \cdot t-c}}} \boldsymbol{x}^{\prime}_{t-1} + \sqrt{1-\frac{\alpha_{c \cdot t}}{\alpha_{c \cdot t-c}}} \boldsymbol{\epsilon}_{t-1} \qquad{;where} \quad \boldsymbol{\epsilon}_{t-1} \sim \mathcal{N}(\mathbf{0}, \boldsymbol{I})  \\
        &= \sqrt{\frac{\alpha_{c \cdot t}}{\alpha_{c \cdot t-c}}}  \left(\sqrt{\frac{\alpha_{c \cdot t -c}}{\alpha_{c \cdot t-2c}}} \boldsymbol{x}^{\prime}_{t-2} + \sqrt{1-\frac{\alpha_{c \cdot t -c}}{\alpha_{c \cdot t-2c}}} \boldsymbol{\epsilon}_{t-2}\right) + \sqrt{1-\frac{\alpha_{c \cdot t}}{\alpha_{c \cdot t-c}}} \boldsymbol{\epsilon}_{t-1} \qquad{;where} \quad \boldsymbol{\epsilon}_{t-2} \sim \mathcal{N}(\mathbf{0}, \boldsymbol{I}) \\
        &= \sqrt{\frac{\alpha_{c \cdot t}}{\alpha_{c \cdot t-2c}}} \boldsymbol{x}^{\prime}_{t-2} + \sqrt{\left(\sqrt{\frac{\alpha_{c \cdot t}}{\alpha_{c \cdot t-c}} -  \frac{\alpha_{c \cdot t}}{\alpha_{c \cdot t-2c}}  }\right)^{2}+ \left(\sqrt{1-\frac{\alpha_{c \cdot t}}{\alpha_{c \cdot t-c}}} \right)^{2}} \bar{\boldsymbol{\epsilon}}_{t-2} \qquad{;where} \quad \bar{\boldsymbol{\epsilon}}_{t-2} \quad {merges} \quad \boldsymbol{\epsilon}_{t-1}, \boldsymbol{\epsilon}_{t-2} \\
        &= \sqrt{\frac{\alpha_{c \cdot t}}{\alpha_{c \cdot t-2c}}} \boldsymbol{x}^{\prime}_{t-2} + \sqrt{1-\frac{\alpha_{c \cdot t}}{\alpha_{c \cdot t - 2c}}} \bar{\boldsymbol{\epsilon}}_{t-2} \\
        &= {...} \\
        &= \sqrt{\alpha_{c \cdot t}}\boldsymbol{x}^{\prime}_0 + \sqrt{1-\alpha_{c \cdot t}} \boldsymbol{\epsilon},
    \end{aligned}
    \label{eq:lemma1_app}
\end{equation}
where we apply the Gaussian property that two distributions $\mathcal{N}(\mathbf{0}, \sigma_{1}^2\boldsymbol{I})$ and $\mathcal{N}(\mathbf{0}, \sigma_{2}^2\boldsymbol{I})$ can be merged as one distribution $\mathcal{N}(\mathbf{0}, (\sigma_{1}^2+\sigma_{2}^2)\boldsymbol{I})$. Then according to Eq.~(\ref{eq:lemma1_app}), we have $ q^{\prime}(\boldsymbol{x}^{\prime}_t \mid \boldsymbol{x}^{\prime}_0) =  \mathcal{N}\left(\boldsymbol{x}^{\prime}_t ; \sqrt{\alpha_{c \cdot t}} \boldsymbol{x}^{\prime}_0,\left(1-\alpha_{c \cdot t}\right) \boldsymbol{I}\right)$.
\end{proof}

\subsection{Derivation of the posterior $q^{\prime}\left(\boldsymbol{x}^{\prime}_{t-1} \mid \boldsymbol{x}^{\prime}_t, \boldsymbol{x}^{\prime}_0\right)$}
\label{app:sec2}

In the following, we show the derivation of 
\begin{equation*}
\begin{aligned}
    q^{\prime}\left(\boldsymbol{x}^{\prime}_{t-1} \mid \boldsymbol{x}^{\prime}_t, \boldsymbol{x}^{\prime}_0\right) =  \mathcal{N}\left(\boldsymbol{x}^{\prime}_{t-1} ; \frac{(1-\alpha_{c \cdot t-c})\sqrt{\alpha_{c \cdot t}}}{(1-\alpha_{c \cdot t})\sqrt{\alpha_{c \cdot t-c}}}\boldsymbol{x}^{\prime}_t + \frac{\alpha_{c \cdot t-c}-\alpha_{c \cdot t}}{(1-\alpha_{c \cdot t})\sqrt{\alpha_{c \cdot t-c}}}\boldsymbol{x}^{\prime}_0, \sigma^{\prime}_t \boldsymbol{I}\right),
\end{aligned}
\end{equation*}
where $\sigma^{\prime}_t = \frac{(1-\alpha_{c \cdot t-c})(\alpha_{c \cdot t-c}-\alpha_{c \cdot t})}{(1-\alpha_{c \cdot t})\alpha_{c \cdot t-c}}$.

\begin{proof}
According to the Bayes' rule $q^{\prime}\left(\boldsymbol{x}^{\prime}_{t-1} \mid \boldsymbol{x}^{\prime}_t, \boldsymbol{x}^{\prime}_0\right)=q^{\prime}\left(\boldsymbol{x}^{\prime}_t \mid \boldsymbol{x}^{\prime}_{t-1}, \boldsymbol{x}^{\prime}_0\right) \frac{q^{\prime}\left(\boldsymbol{x}^{\prime}_{t-1} \mid \boldsymbol{x}^{\prime}_0\right)}{q^{\prime}\left(\boldsymbol{x}^{\prime}_t \mid \boldsymbol{x}^{\prime}_0\right)}$ and the Markovian assumption that $q^{\prime}\left(\boldsymbol{x}^{\prime}_t \mid \boldsymbol{x}^{\prime}_{t-1}, \boldsymbol{x}^{\prime}_0\right) = q^{\prime}\left(\boldsymbol{x}^{\prime}_t \mid \boldsymbol{x}^{\prime}_{t-1}\right)$, we have

\begin{equation*}
\begin{aligned}
    q^{\prime}&\left(\boldsymbol{x}^{\prime}_{t-1} \mid \boldsymbol{x}^{\prime}_t, \boldsymbol{x}^{\prime}_0\right) =q^{\prime}\left(\boldsymbol{x}^{\prime}_t \mid \boldsymbol{x}^{\prime}_{t-1}\right) \frac{q^{\prime}\left(\boldsymbol{x}^{\prime}_{t-1} \mid \boldsymbol{x}^{\prime}_0\right)}{q^{\prime}\left(\boldsymbol{x}^{\prime}_t \mid \boldsymbol{x}^{\prime}_0\right)} \\
    &\propto \exp \left(-\frac{1}{2}\left(\frac{\left(\boldsymbol{x}^{\prime}_{t}-\sqrt{\alpha_{c \cdot t}/\alpha_{c\cdot t-c}} \boldsymbol{x}^{\prime}_{t-1}\right)^2}{1-\frac{\alpha_{c\cdot t}}{\alpha_{c\cdot t-c}}}+\frac{\left(\boldsymbol{x}^{\prime}_{t-1}-\sqrt{\alpha_{c \cdot t-c}} \boldsymbol{x}^{\prime}_0\right)^2}{1-\alpha_{c \cdot t-c}}-\frac{\left(\boldsymbol{x}^{\prime}_t-\sqrt{\alpha_{c\cdot t}} \boldsymbol{x}^{\prime}_0\right)^2}{1-\alpha_{c\cdot t}}\right)\right) \\
    &= \exp \left( -\frac{1}{2}\left(\left(\frac{\alpha_{c \cdot t}}{\alpha_{c \cdot t-c}-\alpha_{c \cdot t}}+\frac{1}{1-\alpha_{c \cdot t-c}}\right) {\boldsymbol{x}^{\prime}_{t-1}}^2-\left(\frac{2 \sqrt{\alpha_{c \cdot t}{\alpha_{c \cdot t-c}}}}{\alpha_{c \cdot t-c}-\alpha_{c \cdot t}} \boldsymbol{x}^{\prime}_t+\frac{2 \sqrt{\alpha_{c \cdot t-c}}}{1-\alpha_{c \cdot t-c}} \boldsymbol{x}^{\prime}_0\right) \boldsymbol{x}^{\prime}_{t-1} + C\left(\boldsymbol{x}^{\prime}_t, \boldsymbol{x}^{\prime}_0\right) \right ) \right)
\end{aligned}
\end{equation*}

Letting $A = \frac{\alpha_{c \cdot t}}{\alpha_{c \cdot t-c}-\alpha_{c \cdot t}}+\frac{1}{1-\alpha_{c \cdot t-c}}$ and $B = \frac{2 \sqrt{\alpha_{c \cdot t}{\alpha_{c \cdot t-c}}}}{\alpha_{c \cdot t-c}-\alpha_{c \cdot t}} \boldsymbol{x}^{\prime}_t+\frac{2 \sqrt{\alpha_{c \cdot t-c}}}{1-\alpha_{c \cdot t-c}} \boldsymbol{x}^{\prime}_0$. Comparing the expression of a Gaussian distribution, we have the variance of $q^{\prime}\left(\boldsymbol{x}^{\prime}_{t-1} \mid \boldsymbol{x}^{\prime}_t, \boldsymbol{x}^{\prime}_0\right)$:
\begin{equation*}
    \sigma^{\prime}_t = \frac{1}{A} = \frac{(1-\alpha_{c \cdot t-c})(\alpha_{c \cdot t-c}-\alpha_{c \cdot t})}{(1-\alpha_{c \cdot t})\alpha_{c \cdot t-c}}.
\end{equation*}
Then, the mean of $q^{\prime}\left(\boldsymbol{x}^{\prime}_{t-1} \mid \boldsymbol{x}^{\prime}_t, \boldsymbol{x}^{\prime}_0\right)$ is
\begin{equation*}
    \frac{B}{A} = \frac{(1-\alpha_{c \cdot t-c})\sqrt{\alpha_{c \cdot t}}}{(1-\alpha_{c \cdot t})\sqrt{\alpha_{c \cdot t-c}}}\boldsymbol{x}^{\prime}_t + \frac{\alpha_{c \cdot t-c}-\alpha_{c \cdot t}}{(1-\alpha_{c \cdot t})\sqrt{\alpha_{c \cdot t-c}}}\boldsymbol{x}^{\prime}_0.
\end{equation*}

\end{proof}

\subsection{Simplifying the student loss}
\label{app:sec3}

Letting $\tilde{\boldsymbol{\mu}}_{\Theta}\left(\boldsymbol{x}_t^{\prime}, \boldsymbol{x}_0^{\prime}\right) = \frac{(1-\alpha_{c \cdot t-c})\sqrt{\alpha_{c \cdot t}}}{(1-\alpha_{c \cdot t})\sqrt{\alpha_{c \cdot t-c}}}\boldsymbol{x}^{\prime}_t + \frac{\alpha_{c \cdot t-c}-\alpha_{c \cdot t}}{(1-\alpha_{c \cdot t})\sqrt{\alpha_{c \cdot t-c}}}\boldsymbol{x}^{\prime}_0$, then we have $$q^{\prime}\left(\boldsymbol{x}^{\prime}_{t-1} \mid \boldsymbol{x}^{\prime}_t, \boldsymbol{x}^{\prime}_0\right) =  \mathcal{N}\left(\boldsymbol{x}^{\prime}_{t-1} ; \tilde{\boldsymbol{\mu}}_{\Theta}\left(\boldsymbol{x}_t^{\prime}, \boldsymbol{x}_0^{\prime}\right), \sigma^{\prime}_t \boldsymbol{I}\right).$$ According to the definition of the backward process of the student we know that $$p_{\Theta}^{\prime}\left(\boldsymbol{x}_{t-1}^{\prime} \mid \boldsymbol{x}_t^{\prime}\right)=\mathcal{N}\left(\boldsymbol{x}_{t-1}^{\prime} ; \boldsymbol{\mu}_{\Theta}^{\prime}\left(\boldsymbol{x}_t^{\prime}, t\right), \sigma_t^{\prime} \boldsymbol{I}\right)$$

Then, referring to Eq.~(\ref{eq:loss_sim}), we have
\begin{equation}
    \begin{aligned}
D_{\mathrm{KL}}\left(q^{\prime}\left(\boldsymbol{x}^{\prime}_{t-1} \mid \boldsymbol{x}^{\prime}_t, \boldsymbol{x}^{\prime}_0\right) \| p_{\Theta}^{\prime}\left(\boldsymbol{x}_{t-1}^{\prime} \mid \boldsymbol{x}_t^{\prime}\right)\right) & =D_{\mathrm{KL}}\left(\mathcal{N}\left(\boldsymbol{x}^{\prime}_{t-1} ; \tilde{\boldsymbol{\mu}}_{\Theta}\left(\boldsymbol{x}_t^{\prime}, \boldsymbol{x}_0^{\prime}\right), \sigma^{\prime}_t \boldsymbol{I}\right) \| \mathcal{N}\left(\boldsymbol{x}_{t-1}^{\prime} ; \boldsymbol{\mu}_{\Theta}^{\prime}\left(\boldsymbol{x}_t^{\prime}, t\right), \sigma_t^{\prime} \boldsymbol{I}\right)\right) \\
& =\frac{1}{2}\left(n+\frac{1}{\sigma_t^2}\left\|\tilde{\boldsymbol{\mu}}_{\Theta}\left(\boldsymbol{x}_t^{\prime}, \boldsymbol{x}_0^{\prime}\right)-       \boldsymbol{\mu}_{\Theta}^{\prime}\left(\boldsymbol{x}_t^{\prime}, t\right)\right\|^2-n+\log 1\right) \\
& =\frac{1}{2 \sigma_t^2}\left\|\tilde{\boldsymbol{\mu}}_{\Theta}\left(\boldsymbol{x}_t^{\prime}, \boldsymbol{x}_0^{\prime}\right)-       \boldsymbol{\mu}_{\Theta}^{\prime}\left(\boldsymbol{x}_t^{\prime}, t\right)\right\|^2.
\end{aligned}
\label{eq:kl_appendix}
\end{equation}
The calculation of the KL-divergence between two Gaussian distributions is referred to \cite{bishop2006pattern}. Then plugging Eq.~(\ref{eq:revers_mean}) and Eq.~(\ref{eq:predictor}) into Eq.~(\ref{eq:kl_appendix}), we have the simplified loss

\begin{equation*}
    L(\Theta):= \sum_{t=1}^{T} \mathbb{E}_{\boldsymbol{x}^{\prime}_0, \boldsymbol{\epsilon}^{\prime}_t}\left[\gamma^{\prime}_t\left\|\boldsymbol{\epsilon}^{\prime}_t-\boldsymbol{\epsilon}_\Theta \left(\sqrt{\alpha_{c\cdot t}} \boldsymbol{x}^{\prime}_0+\sqrt{1-\alpha_{c \cdot t}} \boldsymbol{\epsilon}^{\prime}_t, t\right)\right\|^2\right].
\end{equation*}

\section{Flexible sub-sequence}
\label{app:sec4}

When we train the student by any given subset $\{\boldsymbol{x}_{\phi_0},...,\boldsymbol{x}_{\phi_{T^{\prime}}}\}$ of the teacher where $\phi$ is an increasing sub-sequence of $\{0,...,T\}$, $\phi_0 = 0$ and $\phi_{T^{\prime}} = T$, we take the following form of $q^{\prime}\left(\boldsymbol{x}^{\prime}_t \mid \boldsymbol{x}^{\prime}_{t-1}\right)$ in the forward process of the student:

\begin{equation*}
\footnotesize
q^{\prime}\left(\boldsymbol{x}^{\prime}_t \mid \boldsymbol{x}^{\prime}_{t-1}\right):=\mathcal{N}\left(\boldsymbol{x}^{\prime}_t; \sqrt{\frac{\alpha_{\phi_t}}{\alpha_{\phi_{t-1}}}} \boldsymbol{x}^{\prime}_{t-1},\left(1-\frac{\alpha_{\phi_t}}{\alpha_{\phi_{t-1}}}\right) \boldsymbol{I}\right).
\end{equation*}

Then we have the corresponding property that

\begin{equation*}
q^{\prime}\left(\boldsymbol{x}_t^{\prime} \mid \boldsymbol{x}_0^{\prime}\right)=\mathcal{N}\left(\boldsymbol{x}_t^{\prime} ; \sqrt{\alpha_{\phi_t}} \boldsymbol{x}_0^{\prime},\left(1-\alpha_{\phi_t}\right) \boldsymbol{I}\right),
\end{equation*}
which ensures that
\begin{equation*}
q^{\prime}(\boldsymbol{x}^{\prime}_t = \boldsymbol{x}_{\phi_t} \mid \boldsymbol{x}^{\prime}_0 = \boldsymbol{x}_0) = q(\boldsymbol{x}_{\phi_t}|\boldsymbol{x}_0).
\end{equation*}

The remainder derivation of the student loss of this case is similar with the special $\{\boldsymbol{x}_{c\cdot t}\}$ case detailedly introduced in this paper. The final simplified training loss for this general case is 

\begin{equation*}
\begin{aligned}
    L(\Theta):= \sum_{t=1}^{T} \mathbb{E}_{\boldsymbol{x}^{\prime}_0, \boldsymbol{\epsilon}^{\prime}_t}[\|\boldsymbol{\epsilon}_\theta\left(\sqrt{\alpha_{\phi_t}} \boldsymbol{x}^{\prime}_0+\sqrt{1-\alpha_{\phi_t}} \boldsymbol{\epsilon}^{\prime}_{t}, \phi_t\right) 
    -\boldsymbol{\epsilon}_\Theta \left(\sqrt{\alpha_{\phi_t}} \boldsymbol{x}^{\prime}_0+\sqrt{1-\alpha_{\phi_t}} \boldsymbol{\epsilon}^{\prime}_t, t\right)\|^2].
\end{aligned}
\end{equation*}

\section{\alg on Tabular data generation: 2D Swiss Roll}

\label{app_sec:tabular}

We distill a teacher model with 500 steps into a student with 50 steps using our \alg. We implement \alg on the diffusion model setting shown in a public repository~\footnote{https://github.com/joseph-nagel/diffusion-demo/blob/main/notebooks/swissroll.ipynb}. On this tabular DDPM, the forward process turns Swiss Roll-like points into randomly distributed 2D points. Contrarily, the reverse process constructs a Swiss Roll distribution according to randomly distributed points. The results visualized in Figure~\ref{fig:swiss2d} and Figure~\ref{fig:2dswiss222} demonstrate that our algorithm works on DDPM for distilling the generation of tabular data. 

\begin{figure}[ht]
	\centering
	{
	\subfloat[Teacher diffusion forward process of 500 steps on 2D Swiss Roll dataset.]{
	    \label{fig:swiss2d}
	    \includegraphics[width=0.8\textwidth]{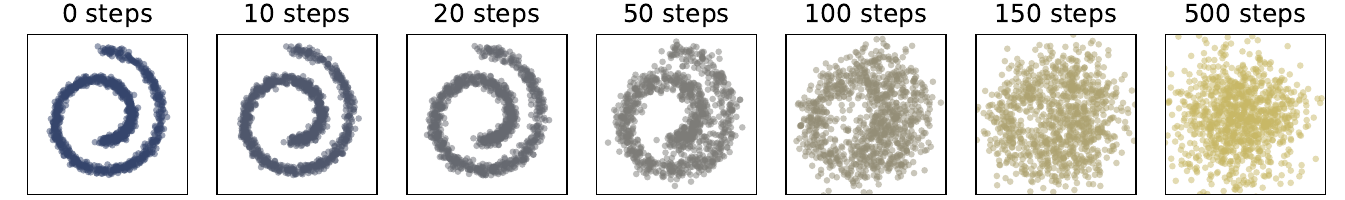} }
     \hfill
     \subfloat[Student diffusion forward process of 50 steps on 2D Swiss Roll dataset.]{
	    \label{fig:vgan-mnist}
	    \includegraphics[width=0.8\textwidth]{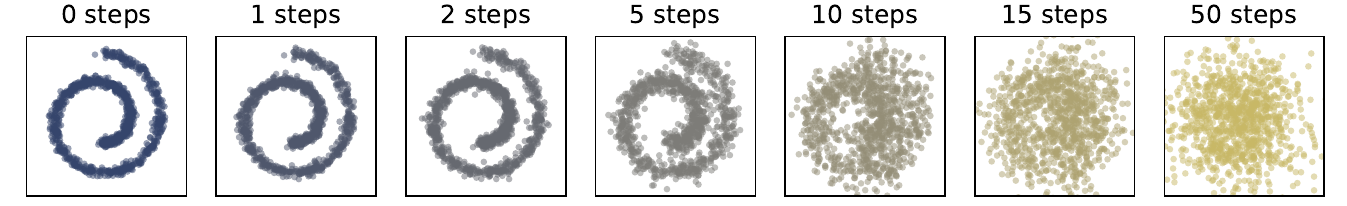} }
     \hfill
     \subfloat[Teacher diffusion reverse process of 500 steps on 2D Swiss Roll dataset.]{
	    \label{fig:vgan-mnist}
	    \includegraphics[width=0.8\textwidth]{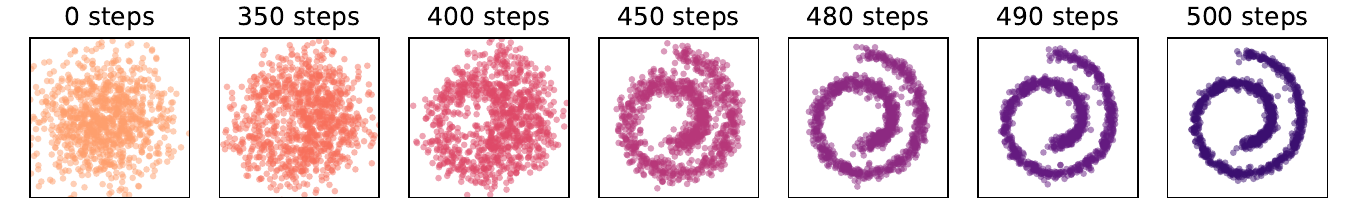} }
     \hfill
     \subfloat[Student diffusion reverse process of 50 steps on 2D Swiss Roll dataset.]{
	    \label{fig:vgan-mnist}
	    \includegraphics[width=0.8\textwidth]{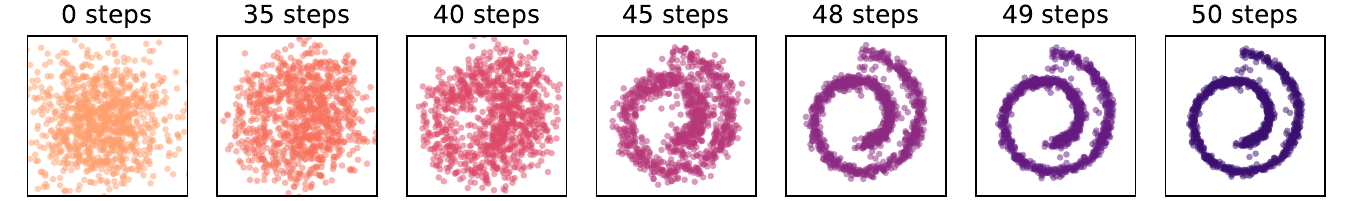} }
	}
	\caption{Forward and reverse process of teacher and \alg student model on 2D Swiss Roll.}
        \label{fig:fid-slayers-cgan}

\vspace{-1.5em}
\end{figure}

\begin{figure}[h!]
	\centering
	{
	\subfloat[Sampled data by the teacher model.]{
	    \label{fig:vgan-mnist}
	    \includegraphics[width=0.4\textwidth]{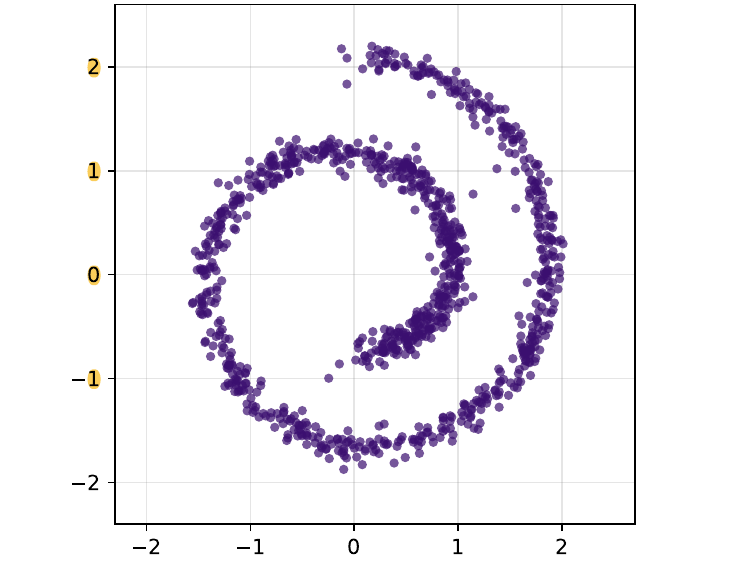} }
     \hspace{-2em}
     \subfloat[Sampled data by the \alg student model.]{
	    \label{fig:vgan-mnist}
	    \includegraphics[width=0.4\textwidth]{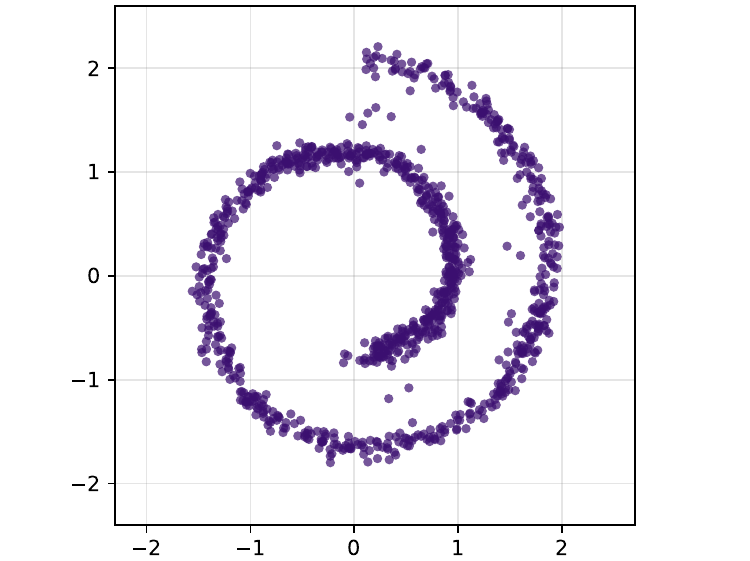} }
	}
	\caption{Generative performance of the teacher and \alg student model on 2D Swiss Roll dataset.}
        \label{fig:2dswiss222}

\end{figure}

\section{Experimental details}
\label{app:exp}
Our evaluation is carried out by Alienware-Aurora-R13 with Ubuntu 20.04. The machine is equipped with 64G memory, $4 \times$ GeForce RTX 3090 GPU and 16-core Intel i9 CPU. Each of the 8 P-cores has two threads, hence each machine contains 24 logical CPU cores in total. 
We consider various image generation benchmarks (CIFAR-10, CelebA, LSUN-Bedroom, LSUN-Church), with resolution varying from $32 \times 32$ to $128 \times 128$. All experiments for the teacher diffusion model use the sigmoid schedule, particularly good for large images, following the settings of ~\cite{sigmoid:conf/icml/JabriFC23} and all models use a UNet architecture same as DDPM~\cite{ddpm:ho2020denoising}.
Our schedule of the student model is defined in Sec.~\ref{subsec:forward} accordingly.  
Our training setup closely matches the open source code by DDPM.
For the training of the student model, we choose Adam optimizer with learning rate fixed to $2 \times 10 ^{-5}$ while other hyper-parameters remain the same as the default setting of PyTorch Adam. An interesting observation on \alg is that experimentally using $l_1$ norm on $L_{\Theta}$ leads to faster convergence comparing to $l_2$ norm. FID scores are computed across 10K images.

For the training of our student model, referring to Alg.~\ref{alg:training}, each loop from Line 3 to Line 8 is regarded as one step for a full batch. We provide the number of steps on LSUN-Bedroom dataset (batch size 40) that one student model arrives converge to demonstrate the efficiency intuitively in Tab.~\ref{tab:converge}. \alg achieves better FID with few training steps comparing to progressive distillation.

\begin{table}
\centering
\caption{The number of steps to arrive convergence of training the student model}  \label{tab:converge}
  \begin{tabular}{c|c|c}
    \toprule
    Method & Progressive & \alg\\
    \midrule
    Steps & 0.23M & 0.094M \\

  \bottomrule
\end{tabular}
\end{table}

\begin{table*}[h!]
\renewcommand\arraystretch{1.2}
\centering
\caption{FID of the teacher $T=1024$ on four datasets.}
  \label{tab:fid_teacher}
\resizebox{0.8\columnwidth}{!}{%
\begin{tabular}{c|c|c|c|c} 
\toprule
\textbf{Dataset}& \textbf{Cifar-10} & \textbf{CelebA-HQ}&
\textbf{LSUN-Bedroom} & \textbf{LSUN-Church} \\
\midrule
\textit{Teacher} & 2.49 & 4.05 & 2.34 & 2.25\\

\bottomrule
\end{tabular}
}
\end{table*}

\section{Discussion}

We discuss the characteristics of different distillation methods. Progressive distillation trains a student model by a progressive halving manner so as to reduce the gap between the teacher and the student in each halving. However, this specific binary distillation way limits its scope of application, e.g., the number of steps of the teacher should be a power of 2. Consistency model is a new type of diffusion models and it supports single-step generation at its intrinsic design. But the distilled student is no longer the same model type as the teacher. \alg is proposed for distilling any DDPM-like teacher models (including image DDPM and tabular DDPM, etc) while the distilled student model is still a DDPM with the same model type as the teacher. Besides, on the same sampling inputs, \alg student has consistent outputs and interpolation with the teacher because the student Markov chains are simplified copies of the teacher's Markov chains.


\section{Additional results}

\begin{figure*}
	\centering
	{
    \includegraphics[width=0.8\textwidth]{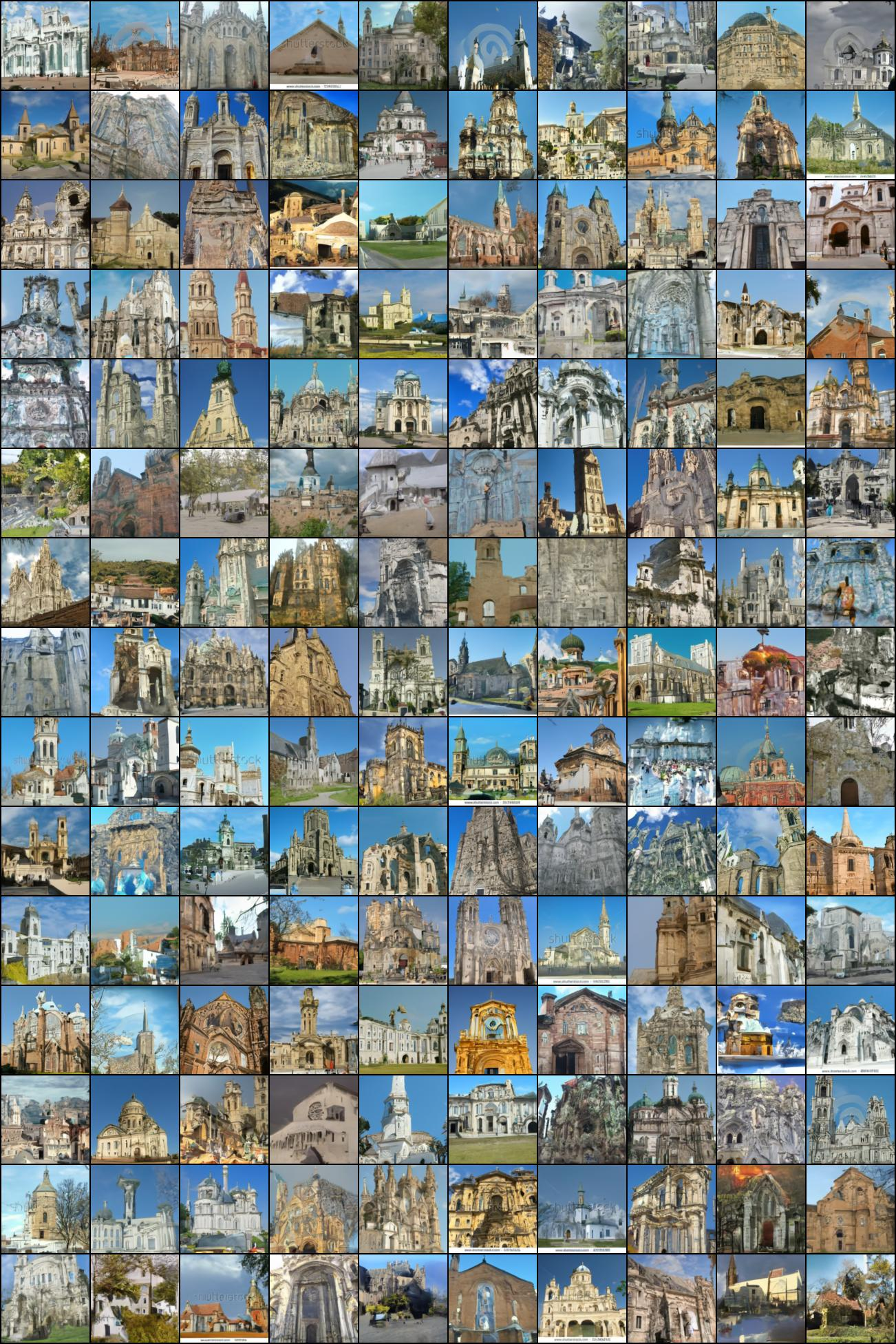} }
	\caption{Sampling image by our \alg student model on LSUN-Church dataset with 100 steps.}
\end{figure*}

\begin{figure*}
	\centering
	{
    \includegraphics[width=0.8\textwidth]{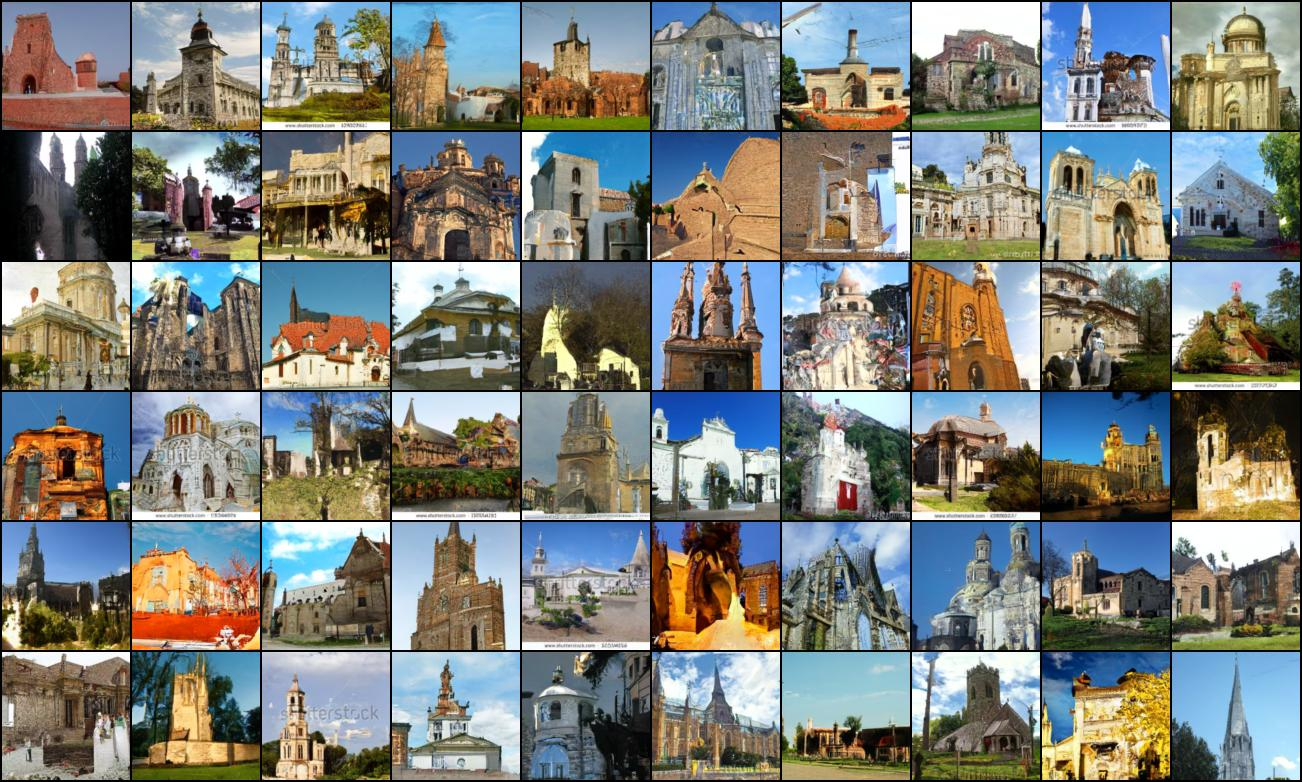} }
	\caption{Sampling image by the teacher with 1024 steps on LSUN-Church dataset}
\end{figure*}

\begin{figure*}
	\centering
	{
    \includegraphics[width=0.8\textwidth]{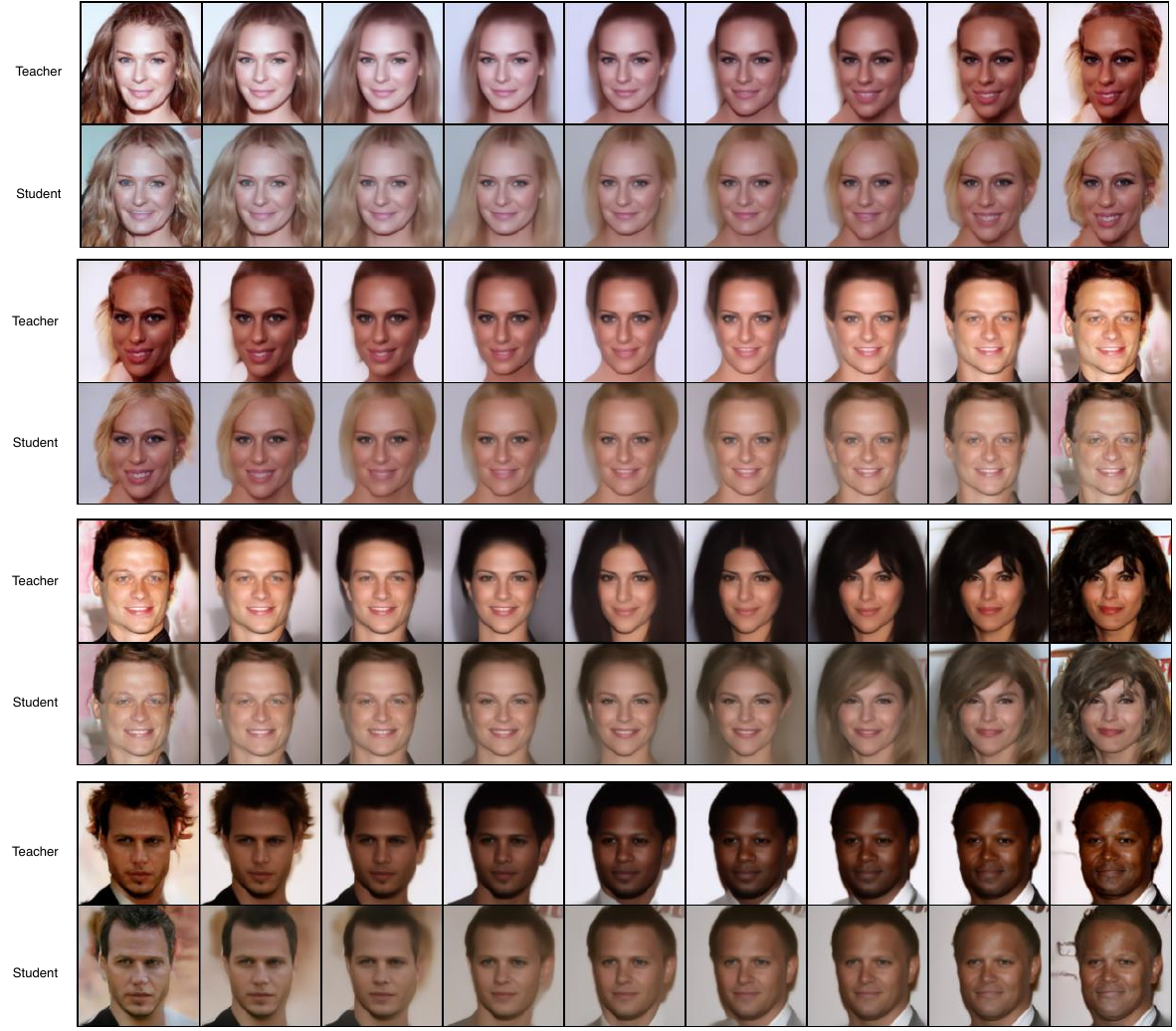} }
	\caption{Interpolation by our \alg student model compared with teacher diffusion model on CelebA dataset with 100 steps out of 1024 steps.}
        \label{fig:church}
\end{figure*}
\clearpage

\end{document}